\documentclass[lettersize,journal]{IEEEtran}
\usepackage{amsmath,amsfonts}
\usepackage{array}
\usepackage[caption=false,font=footnotesize,labelfont=rm,textfont=rm]{subfig}
\usepackage{textcomp}
\usepackage{stfloats}
\usepackage{url}
\usepackage{verbatim}
\usepackage{graphicx}
\usepackage{cite}
\usepackage{color}
\usepackage{hyperref}
\usepackage{etoolbox}
\usepackage{pifont} 
\usepackage{bm}
\usepackage{cases}
\usepackage{indentfirst}
\usepackage{subeqnarray}
\usepackage{amsthm}

\usepackage{type1cm}

\newtheorem{lemma}{Lemma}
\graphicspath{ {./images/} }
\usepackage[linesnumbered,ruled,vlined]{algorithm2e}
\usepackage{algpseudocode}

\begin{document}

\title{Hierarchical Federated Learning for Social Network with Mobility}

\author{{Zeyu Chen, Wen Chen,~\IEEEmembership{Senior~Member,~IEEE}, Jun Li,~\IEEEmembership{Fellow,~IEEE},\\
Qingqing Wu,~\IEEEmembership{Senior~Member,~IEEE}, Ming Ding,~\IEEEmembership{Senior~Member,~IEEE}, \\Xuefeng Han, Xiumei Deng, Liwei Wang}
\thanks{Zeyu Chen, Wen Chen, Qingqing Wu, Xuefeng Han and Liwei Wang are with the Department of Electronic Engineering, Shanghai Jiao Tong University, Minhang 200240, China (e-mail: \{xiaoyuc2001; wenchen; qingqingwu; hansjell-watson; wanglw2000\}@sjtu.edu.cn). Jun Li is with the School of Information Science and Engineering, Southeast University, Nanjing, 210096, CHINA.(email: jleesr80@gmail.com). Ming Ding is with Data61, CSIRO, Sydney, NSW 2015, Australia (e-mail: ming.ding@data61.csiro.au). Xiumei Deng is with the Pillar of Information Systems Technology and Design, Singapore University of Technology and Design, Singapore 487372 (e-mail: xiumei deng@sutd.edu.sg). Correspondence author: Wen Chen
}
}


\maketitle

\begin{abstract} 
Federated Learning (FL) offers a decentralized solution that allows collaborative local model training and global aggregation, thereby protecting data privacy. In conventional FL frameworks, data privacy is typically preserved under the assumption that local data remains absolutely private, whereas the mobility of clients is frequently neglected in explicit modeling. In this paper, we propose a \textit{hierarchical federated learning} framework based on the \textit{social network with mobility} namely HFL-SNM that considers both data sharing among clients and their mobility patterns. Under the constraints of limited resources, we formulate a joint optimization problem of resource allocation and client scheduling, which objective is to minimize the energy consumption of clients during the FL process. In social network, we introduce the concepts of Effective Data Coverage Rate and Redundant Data Coverage Rate. We analyze the impact of effective data and redundant data on the model performance through preliminary experiments. We decouple the optimization problem into multiple sub-problems, analyze them based on preliminary experimental results, and propose Dynamic Optimization in Social Network with Mobility (DO-SNM) algorithm. Experimental results demonstrate that our algorithm achieves superior model performance while significantly reducing energy consumption, compared to traditional baseline algorithms.  
\end{abstract} 

\begin{IEEEkeywords}
Hierarchical Federated Learning, User selection, Resource Allocation, Social Network, Mobility.
\end{IEEEkeywords}

\section{Introduction}
With the breakthroughs in artificial intelligence, particularly in artificial neural networks, Machine Learning (ML) has rapidly evolved into a transformative technology. John J. Hopfield \cite{hopfield1982neural} and Geoffrey E. Hinton \cite{rumelhart1986learning} were awarded the 2024 Nobel Prize in Physics for their innovative works on physics by ML, which shows the critical role of  ML and artificial intelligence in contemporary science and technology. With the deployment of edge computing\cite{7488250} and the rapid growth of social media, massive data streams are generated, heightening public concern over data privacy and security. However, traditional ML models typically require data to be aggregated into central server, which poses privacy risks and is inefficient, especially for security-critical and privacy-sensitive data \cite{7342881,8110700,7883893}. This trend has led to the emergence of Federated Learning (FL) \cite{pmlr-v54-mcmahan17a} as an alternative model. FL facilitates decentralized model training on edge devices without sharing raw data \cite{LI201776}, which effectively mitigates the risk of privacy data leakage. As the advantages of FL in privacy protection and distributive training become increasingly apparent, many researchers have conducted in-depth research on FL about privacy protection\cite{9325934,9408373,10058838,9069945,shi2021hfl}. Wei \textit{et al} \cite{9069945} proposed the noising before model aggregation FL algorithm by integrating differential privacy(DP), which effectively balances the relationship between privacy protection and model convergence. However, although FL can effectively address privacy issues, it still encounters many critical challenges, such as excessive energy consumption by clients, high communication latency, and poor model performance, which have become critical areas of investigation. Chen \textit{et al}. \cite{9330566} proposed a multi-task FL framework in the context of Multi-Access Edge Computing networks  to minimize communication delays, addressing the challenges posed by the high mobility of terminal devices. Yang \textit{et al}. and Mo \textit{et al}. \cite{9264742,9475121} designed efficient algorithms to address the energy minimization problem under communication delay constraints in FL systems. \cite{9207871, 9210812, 9676703, 9772007, 10666737, 10375295} focus primarily on improving model performance with maximum accuracy under conditions of limited communication delay or energy constraints. These studies provide effective and feasible solutions for traditional FL in reducing overhead and improving model performance. 

Traditional FL relies on direct communication between devices and a central server. When a large number of devices participate, communication costs will increase dramatically. This issue becomes particularly pronounced in user-dense areas, where large-scale user access and model uploads strain network bandwidth. Furthermore, the centralized communication model in FL places a heavy burden on the server, which may struggle to handle model updates from all devices simultaneously. This leads to increased training latency and resource consumption, further slowing down the convergence of the model. In order to break through the bottlenecks of the traditional FL, \cite{liu2019clientedgecloudhierarchicalfederatedlearning,wang2020local,9054634} proposed a Hierarchical FL (HFL) architecture, which comprises three key components: Clients, Edges and Clouds. The introduction of edge server effectively alleviates the communication burden on the central server. \cite{liu2019clientedgecloudhierarchicalfederatedlearning} and \cite {wang2020local} demonstrated that HFL achieves
faster convergence than traditional FL architectures. However, these initial research only rely on a pre-defined system architecture and overlook the key challenge of heterogeneous communication costs. To address this limitation, Chai \textit{et al}. \cite{chai2020tifl} introduced a tier-based FL system (TiFL) that categorizes nodes into tiers based on their training speed and selects nodes with similar speeds for each training round to alleviate the straggler issue. Nevertheless, TiFL primarily focuses on optimizing training time rather than addressing the communication cost. Furthermore, Lim \textit{et al}. \cite{lim2021dynamic} tackled the dynamic edge association and resource allocation issues in HFL by proposing a hierarchical game framework. Similarly, Luo \textit{et al}. \cite{luo2020hfel} proposed Hierarchical Federated Edge Learning to optimize edge association and resource allocation, aiming to minimize the learning cost in HFL systems. However, the above approaches primarily focus on network bandwidth while neglecting the heterogeneous data distribution such as non-IID among the distributed nodes. Deng \textit{et al}. \cite{9999679}, and Wei \textit{et al}. \cite{9609994} proposed the  dynamic device scheduling and resource allocation and Greedy Matching with a Better Alternative optimization algorithms based on participation rate \cite{lyu2019optimal,xia2020multi,huang2020efficiency}, effectively achieving a balance between minimizing training delay and optimizing model performance in non-IID scenarios.

Nevertheless, the above studies are based on a static FL framework, where the locations of clients remain fixed throughout. However, in practical scenarios, clients exhibit mobility, which results in variations in channel conditions. Shi \textit{et al}. \cite{9207871} simulated client mobility by randomly initializing client positions at the beginning of each global round. In contrast, Chen \textit{et al}. \cite{chen2020performance} investigated a more realistic FL system, highlighting that client mobility during training can lead to clients moving out of the base station coverage area. Fan \textit{et al}.\cite{fan2023mobility} proposed a practical client mobility model in FL across multiple base stations, which takes into account the speed and direction of clients' movements, rather than relying on random initialization. 

On the other hand, existing research on FL all assumes that each data of client is entirely privately owned, without intersection with the data of other clients. In reality, some data may be shared among different clients. For instance, in daily communication, messages sent via phone calls, text messages, and emails are stored on both sender and receiver. Similarly, conversations on platforms like Twitter and WeChat are visible to participants involved. These data may contain text, images, and videos, which are closely tied to personal privacy and often rely on FL for model training, such as for detecting emotions or facilitating rapid responses. Given that each data sample is stored on both devices of sender and receiver, if both devices participate in FL, the shared portion of their data is retrained one time during a single global round. As the number of clients expands, this leads to a complex structure of a social network graph in which each client is represented as a node, and the data shared between clients serves as undirected edges connecting these nodes.\footnote{As a preliminary exploration of this scenario, this paper only considers data sharing between at most two clients. For interaction scenarios involving multiple participants such as group chats, which would require the construction of hypergraphs where identical hyperedges exist among multiple nodes, these cases fall outside the scope of this study.} However, allowing all clients to participate in FL results in a significant amount of data being retrained in a single global round, which inevitably increases the latency and energy consumption of clients. We are particularly interested in whether such repeated training affects the model performance of FL. Thus, we conducted preliminary experiments, which are shown in Appendix A. Previous research on FL typically does not consider the aspect of data sharing. We are the first to integrate this characteristic by constructing a social network graph model, designed to explore optimization problems within the context of FL. Moreover, the hierarchical FL structure and client mobility are considered in this paper. The main contributions of this paper are summerized as follows:
\begin{itemize}
    \item [\textbullet] We introduce the concepts of Effective Data Coverage Rate (EDCR) and Redundant Data Coverage Rate (RDCR) in the social network graph. According to the results of preliminary experiments, we validated EDCR is positively correlated with model performance, whereas RDCR shows the opposite trend.
    \item [\textbullet] Considering the communication between clients and servers, we propose a HFL framework based on the social network with mobility (HFL-SNM). Constrained by communication latency, bandwidth and EDCR, we formulate a joint optimization problem on client selection, edge association, CPU frequency controlling and bandwidth allocation to minimize the energy consumption of clients. We incorporate differential privacy to safeguard client privacy by injecting Gaussian noise during both model aggregation and broadcasting phases.
    \item [\textbullet]We decouple the problem into three sub-problems: resource allocation, edge association, and client selection. We prove the resource allocation problem is convex and design an Alternating Optimization (AO) algorithm to obtain the optimal bandwidth and CPU frequency of each client. Edge association is solved via Fast Greedy (FG) algorithm based on AO. As for client selection, combining the derived upper bound of objective function with the result of preliminary experiments, we design an Performance-Energy Metric Optimization (PEMO) algorithm. We develop the DO-SNM algorithm combining these methods.
    \item [\textbullet]Through experiments, we investigate the performance of DO-SNM under different EDCR constraints and identify the optimal EDCR constraint value. Compared with four baseline algorithms, DO-SNM reduces energy consumption by at least $60\%$, achieving at least $1\%$ higher test accuracy.
\end{itemize}

The rest of this paper is organized as follows. In Section II, we introduce the the system model and formulate the optimization problem. The Differential Privacy mechanism for HFL-SNM is proposed in Section III. The problem is decoupled and analyzed in Section IV, where the DO-SNM Algorithm is proposed. Then, we present the evaluation results in Section V and conclude this paper in Section VI. The preliminary experiments are exhibited in Appendix A. The prove of lemma 1-7 in Section III is shown in Appendix B-H.
 
\section{System Model And Problem Formulation }
We consider the framework in Fig. 2 consists of $M$ clients, $K$ edge servers (ESs) and a cloud server (CS). In this framework, we assume that a social network consists of clients $\mathcal{U}=\{u_{1},u_{2},...,u_{N}\}$, which is substitute by a community in Fig.2. Each client owns a local data set $\mathcal D_n = \{ (\bm x_i, y_i) | i=1, 2, \cdots, D_n \}$, where $\boldsymbol{x}_i$ denotes the data feature of the $i$-th sample, $y_i$ denotes the corresponding labeled output of $\boldsymbol{x}_i$, and ${D}_n$ denotes the data size of $\mathcal{D}_n$. Let ${
D}=\sum_{n \in \{i| u_i \in \mathcal{U}\}}{D}_n$ denote the volume of total data in $\mathcal{U}$. Let $\mathcal{S}=\{u_{1},u_{2},...,u_{M}\}$ denote the clients selected from the community to participate in FL, which is a subset of $\mathcal{U}$. Let $\mathcal{K=}\{e_{1},e_{2},...,e_{K}\}$ denote the set of ESs. $\mathcal{S}^{\mathrm{k}}$ represents the subset of $\mathcal{S}$, in which clients are associated with  $\mathrm{e}_k$. 

\subsection{Social Network Graph}
According to the data statistics from Facebook and Twitter, the social network graph of a social media is sparse \cite{ravazzi2021learning}. Thus we create a sparse graph to describe the relations among the clients. The graph is shown in Fig. 2, where different nodes represent different clients, and the weight of each edge indicates the size of the shared data between the connected clients. We define effective data held in the selected clients $\mathcal{S}$ as $\mathcal{D}_\mathrm{ef}^\mathcal{S}=\bigcup_{n \in \{i| u_i \in \mathcal{S}^k\}}\mathcal{D}_n$. Let $\mathcal{D}_\mathrm{re}^\mathcal{S} = \bigcup_{i=1}^{M-1} \bigcup_{j=i+1}^{M} \left( \mathcal{D}_i \cap \mathcal{D}_j \right)$ denote the redundant data held in $\mathcal{S}$. The EDCR and RDCR is defined as
\begin{equation}\label{eq:(1)}
 r_\mathrm{ef}=\frac{{D}_{\mathrm{ef}}^{\mathcal{S}}}{{D}_{\mathrm{ef}}^{\mathcal{U}}},
\end{equation}
\begin{equation}\label{eq:(2)}
r_\mathrm{re}=\frac{{D}_{\mathrm{re}}^{\mathcal{S}}}{{D}_{\mathrm{re}}^{\mathcal{U}}},
\end{equation}
where $D_{\mathrm{ef}}^{\mathcal{S}}$ and $D_{\mathrm{re}}^{\mathcal{S}}$ denote the data size of $\mathcal{D}_{\mathrm{ef}}^{\mathcal{S}}$ and $\mathcal{D}_{\mathrm{re}}^{\mathcal{S}}$ respectively.
\begin{figure}
    \centering 
    \includegraphics[height=6cm,width=8.4cm]{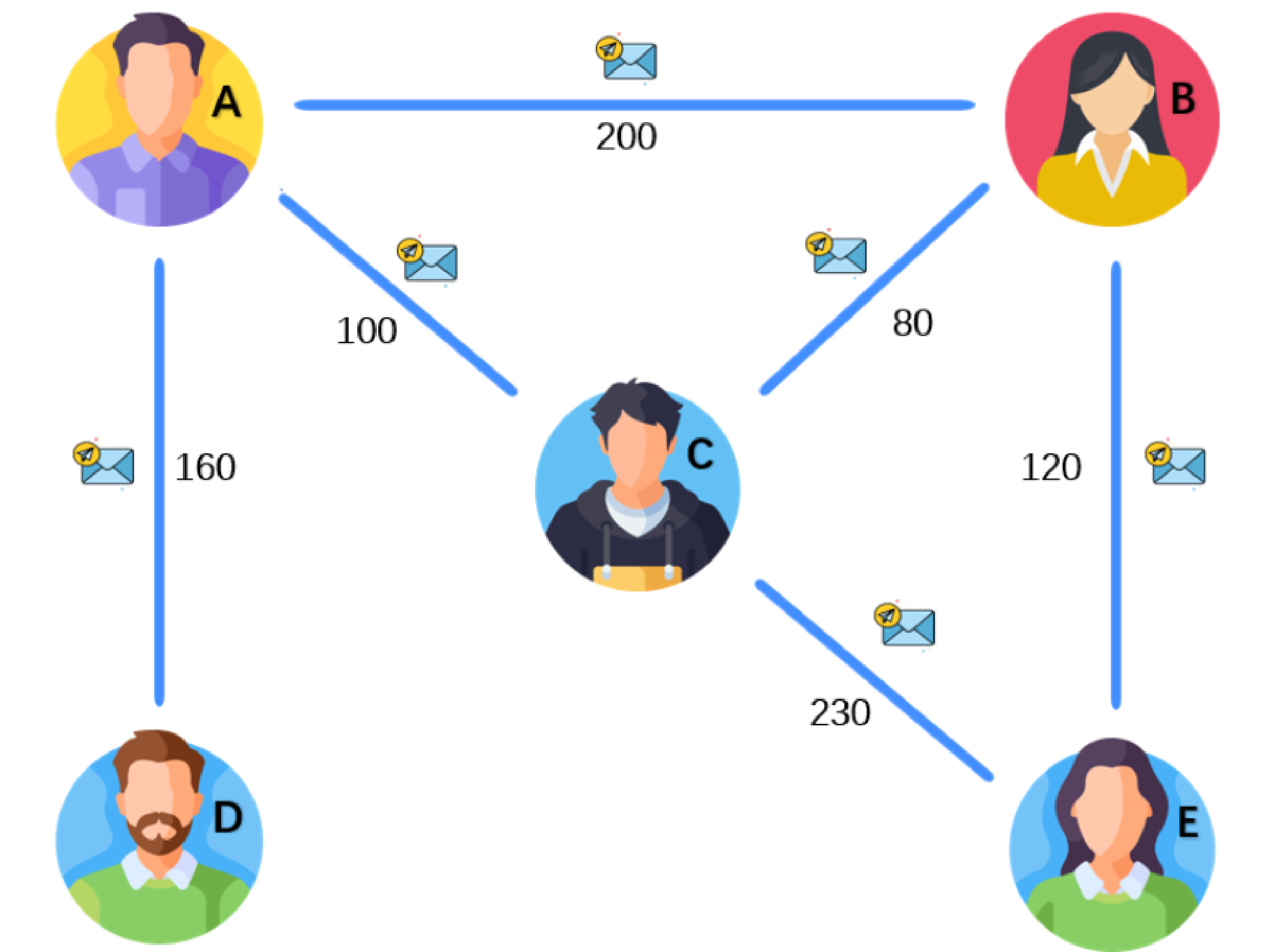}
    \caption{An example of social network. Clients A, B, C, D, E respectively possess 460, 400, 410, 160, 350 samples.  If A, B, C are selected, all 890 (=160+100+200+80+120+230) unique samples of the five clients are covered, which are effective data. While the amount of data to be trained reaches 1270 (=460+400+410), producing additional 380 (=1270-890) samples, which are redundant data. The redundant data consists of 200 samples from A and B, 100 samples from B and C , 80 samples from A and C. If five clients simultaneously selected, the amount of effective data is still 890 while the amount of redundant data reaches 890. If client A individually selected, the amount of effective data is 460 and the redundant data is 0. } \label{Social-Network}
\end{figure}
\subsection{Mobility}
In this paper, we choose Urban Macro (UMa) as the communication scenario, which is suitable for user-dense environments, such as commercial and residential areas in cities. The intersite distance in the UMa scenario is usually around $500$m \cite{3gpp38901}. The communication range of each ES is $2000$m in this paper. ESs are uniformly distributed and their locations remain static.

To realistically simulate the mobility of clients, we adopt the mobility model based on \cite{fan2023mobility} and \cite{3gpp38901}. Let $(x_n^T, y_n^T)$ denote the position of $u_n$ in the T-th global round. At the beginning of the (T+1)-th global round, the position is updated according to the following rules:
\begin{align*}
&x_{n}^{T+1} =x_{n}^{T}+{v_{n}^{T}}\cos{d_n^{T}} t_n^{T},\\
&y_{n}^{T+1} =y_{n}^{T}+{v_{n}^{T}}\sin{d_n^{T}} t_n^{T},
\end{align*}
where ${v_n^{T}} \in [0,100]$ $\mathrm{km/h}$ and $d_n^{T} \in [0,2\pi)$ represent the the magnitude and the direction of the velocity of $u_n$ in the T-th global round. Before training, $(x_n^0, y_n^0)$ and $v_n^0$ are randomly initialized; at the beginning of each global round, $d_n^{T}$ is randomized.

\subsection{System Architecture}
The model framework is shown in Fig. 2. For the selected clients, each client performs a local training to update its local model, after which it uploads the updated local model to its association ES. Then each ES performs edge aggregation on the models uploaded by clients. After a certain number of edge aggregation, all ESs upload the edge models to the CS. The CS performs global aggregation to generate a new global model. The new global model is distributed to ESs to initiate a new round of training.
\begin{figure*}[htbp] 
    \centering
    \includegraphics[width=14cm,height=12cm]{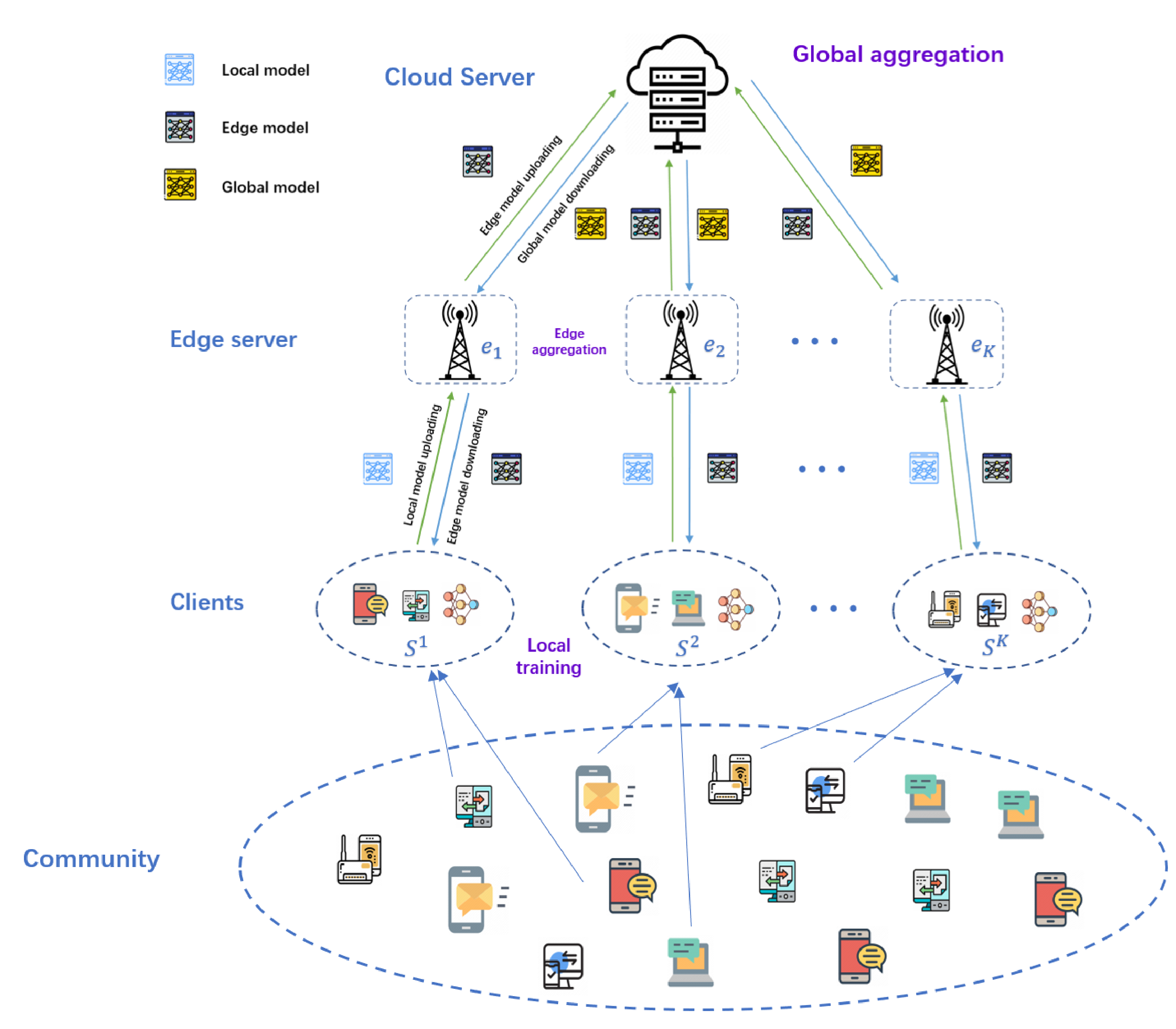} 
    \caption{Hierarchical Federal Learning Framework-Social Network with Mobility.}
    \label{System-Structure}
\end{figure*}

In HFL-SNM framework, both edge aggregation and global aggregation are performed with synchronous aggregation. In addition, the bandwidth and edge iteration of each ES is considered equal. Let $p_n$, $f_n$ and $\lambda_n$ represent the transmission power, transmission frequency, and local iteration number of $u_n$, respectively. Following this overview, the detailed processes of local training, edge aggregation and global aggregation are  described in the following subsections.

\subsubsection{Local training}
In this stage, all selected clients utilize their local data to train the local models. Let $\boldsymbol{\omega}_n^t$ represent the model of $u_n$  at the $t$-th local iteration. Each client utilizes Gradient Descent to perform the model updating. Let $f_n\left(\boldsymbol{\omega}_{n}^t,\boldsymbol{x}_i,y_i \right)$ denote the loss function of $\boldsymbol{x}_i,y_i$. The specific description are as follows. 
\begin{equation}\label{eq:(3)}
 F_n(\boldsymbol{\omega}_{n}^t)=\frac{1}{{D}_n} \sum_{i=1} ^{{D}_n}f_n\left(\boldsymbol{\omega}_{n}^t,\boldsymbol{x}_i,y_i \right),
\end{equation} 
\begin{equation}\label{eq:(4)} 
 \boldsymbol{\omega}_{n}^{t+1}=\boldsymbol{\omega}_{n}^{t}-\eta \nabla F_n\left(\boldsymbol{\omega}_{n}^{t}\right),
\end{equation}
where $\boldsymbol{\omega}_{n}^{t+1}$ denotes the model parameters at the ($t+1$)-th local iteration, and $\eta$ is the predetermined learning rate.

The overhead of local training is divided into two parts: local computation latency and energy consumption, which are firmly related to the data size and the CPU frequency. Limited to the hardware, the CPU frequency of $u_n$ is constrained by $\nu_{\min} \leq \nu_n \leq \nu_{\max}$. Let $C$ denote the number of CPU operations required for one gradient descent iteration on a sample, which is considered equal for all clients. The computation latency and energy consumption can be calculated by
\begin{equation}\label{eq:(5)}
 t_n^\mathrm{cmp}=\frac{\lambda_n C {D}_n}{\nu_n},
\end{equation}  
\begin{equation}\label{eq:(6)}
 E_n^\mathrm{cmp}=\frac{\alpha}{2}\lambda_n C\nu_{n}^2 {D}_n,
\end{equation}
where $\lambda_n$ denotes the number of local iterations bounded by $\mathcal{O}(\kappa \log(\frac{1}{\varepsilon}))$\cite{Konečný03092017}, with $\kappa$ being a coefficient proportional to $D_n$ and $\varepsilon$ representing the local training accuracy threshold. $\alpha$ denotes the chip capacitance coefficient.
\subsubsection{Edge aggregation}
After local updates, all clients send the local models to associated ESs for edge aggregation. 
Let $\boldsymbol{\omega}_{k}^{l+1}$ represents the edge model of $e_k$ in the ($l+1$)-th edge iteration. The edge model is aggregated with weights based on the data size from the clients. The specific update process is represented as
\begin{equation}\label{eq:(7)}
 \boldsymbol{\omega}_{k}^{l+1}=\frac{\Sigma_{n \in \{i| u_i \in \mathcal{S}^k\}} {D}_n \boldsymbol{\omega}_{n}^{\lambda_n}}{\sum_{n \in \{i| u_i \in \mathcal{S}^k\}}{D}_n},
\end{equation}
where $w_{n}^{\lambda_n}$ denotes the local model of $u_n$ in previous edge iteration. At the beginning of the $l$-th edge iteration, $\omega_n^0 = \omega_k^l$. 

We consider the uploading process is implemented via Frequency Division Multiple Access (FDMA). Let $\mathrm{B}_0$ represent the bandwidth of each ES. According to the Shannon formula, the information transmit rate $r_{n,k}$ between $u_n$ and $e_k$ is calculated by
\begin{equation}\label{eq:(8)}
 r_{n,k}=\delta_{n,k}B_{n,k}\log_{2} (1+\frac{h_{n,k}p_n}{B_{n,k}N_{0}}),
\end{equation}
where $B_{n,k}$ denotes the bandwidth of $u_n$ allocated by $e_k$. $\delta_{n,k}\in\{0,1\}$ denotes the association status between $u_n$ and $e_k$. When $u_n\in\mathcal{S}^{\mathrm{k}}$, $\delta_{n,k}=1$, otherwise $\delta_{n,k}=0$. $N_{0}$ is the background noise, and $p_n$ denotes the transmission power\cite{3gpp38101} of $u_n$. $h_{n,k}$ denotes the channel fading coefficient during transmission. Under UMa scenarios in \cite{3gpp38901}, the channel fading is calculated by
\begin{equation}\label{eq:(9)}
 -10\log_{10}(h_{n,k}) = 32.4+20\log_{10}(f_{n})+30\log_{10}(dist_{n,k}),
\end{equation}
where $f_{n}$ denotes the transmission frequency\cite{3gpp38101} of $u_n$ and $dist_{n,k}$ represents the physical distance between $u_n$ and $e_k$. Then the communication latency between $u_n$ and $e_k$ is expressed as 
\begin{equation}\label{eq:(10)}
 t_{n,k}^\mathrm{com}=\frac{z}{r_{n,k}}=\frac{z}{\delta_{n,k}B_{n,k}\log_{2}\left(1+\frac{h_{n,k}p_n}{B_{n,k}N_{0}}\right)},
\end{equation}
where ${z}$ indicates the bit number of the local models updated by clients. On account of the identical neural network structure for clients, all clients have the same ${z}$. Therefore,the energy consumption is calculated by
\begin{equation}\label{eq:(11)}
 E_{n,k}^\mathrm{com}=pt_{n,k}^\mathrm{com}=\frac{zp}{\delta_{n,k}B_{n,k}\log_{2}\left(1+\frac{h_{n,k}p_n}{B_{n,k}N_{0}}\right)}.
\end{equation}
\subsubsection{Global aggregation}
After finishing the edge aggregation, all ESs upload the edge models to CS for global aggregation, which is characterized as
\begin{equation}\label{eq:(12)}
 \boldsymbol{\omega}^{T+1}=\frac{\sum_{k \in \{i| e_i \in \mathcal{K}\}}\sum_{n \in \{i| u_i \in \mathcal{S}^k\}}{D}_n\boldsymbol{\omega}_{k}^{\tau}}{\sum_{n \in \{i| u_i \in \mathcal{S}\}}{D}_n},
\end{equation}
where $\omega_{k}^\tau$ represents the edge model of $e_k$ after $\tau$ edge aggregation. At the beginning of $T$th global round, $\omega_k^0 = \omega^T$.

Typically, the associated latency between CS and ES are negligible compared to local computation and uplink communication \cite{9475121,9207871,9676703}. Then the total latency and energy consumption in a global round is computed by
\begin{equation}\label{eq:(13)}
 t_\mathrm{total}=\max\left\{\tau\left(t_n^{cmp}+t_{n,k}^{com}\right)\right\},
\end{equation}
\begin{equation}\label{eq:(14)}
 E_\mathrm{total}=\sum_{k \in \{i| e_i \in \mathcal{K}\}}\tau\left(\sum_{n \in \{i| u_i \in \mathcal{S}^k\}}(E_n^{cmp}+E_{n,k}^{com})\right).
\end{equation}

\subsection{Problem Formulation}
This study aims to comprehensively consider the social network graph, client mobility, and model performance. With the objective of minimizing user energy consumption, the joint optimization problem is formulated as
\begin{subequations}
\begin{align}\label{eq:(15)}
\textbf{ P:} \phantom{E_{to}}& \min_{\{\mathcal S,\mathcal{S}^k,B_{n,k},\nu_n\}} E_\mathrm{total}, \notag\\
{\rm s.t.} \quad & r_\mathrm{ef}\geq r_\mathrm{ef}^0, \\
& \delta_{n,k} \in \{0, 1\}, \quad \forall u_n \in \mathcal{S}, \thinspace e_k \in \mathcal{K},\\
& \sum_{n \in \{i| u_i \in \mathcal{S}^k\}}\delta_{n,k} = 1, \quad \forall \mathcal{S}^k \subseteq \mathcal{S}, \thinspace e_k \in \mathcal{K},\\
& t_n^\mathrm{cmp}+t_{n,k}^\mathrm{com} \leq t_0, \quad \forall u_n \in \mathcal{S}, \thinspace e_k \in \mathcal{K},\\
& \sum_{n \in \{i| u_i \in \mathcal{S}^k\}} B_{n,k} \leq B_0, \quad \forall e_k \in \mathcal{K},\mathcal{S}^k \subseteq \mathcal{S}, \\
& 0 \leq B_{n,k} \leq B_0, \quad \forall u_n \in \mathcal{S}^k, \thinspace e_k \in \mathcal{K}, \\
& \nu_\mathrm{min}\leq\nu_{n}\leq\nu_\mathrm{max}, \quad \forall u_n \in \mathcal{S}, \\
& \bigcup_{k \in \{i| e_i \in \mathcal{K}\} }\mathcal{S}^k = \mathcal{S}, \quad  \mathcal{S} \subseteq \mathcal{U}. 
\end{align}
\end{subequations}

Constraint (15a) represents the EDCR requirement, ensuring that the final model achieves satisfactory performance by maintaining adequate effective data. Constraint (15b) and (15c) are indicator functions that represent the association status between clients and ESs, determining whether a connection is established between them. Constraint (15d) imposes a latency restriction, representing the maximum tolerable delay time for each ES, ensuring that both the transmission rate and latency overhead are within acceptable limits. Constraints (15e) and (15f) establish the bandwidth allocation restrictions on the associated clients. Constraint (15g) characterizes the constraints on computational resources for selected clients. Constraint (15h) represents the relationships between all clients in the social network, the selected clients, and the clients allocated to different ESs.

\section{Differential Privacy For HFL-SNM}
Differential Privacy is a rigorous theoretical framework for privacy protection. Its essence lies in adding carefully crafted random noise to the data processing procedure, guaranteeing that the presence or absence of any single individual record has an almost negligible impact on the probability distribution of the final output. More precisely, for any two neighboring datasets $\mathcal{D}$ and $\mathcal{D}'$, differing by a single record, and any possible output set $\mathcal{R}$, a mechanism $\mathcal{M}$ satisfies $(\epsilon,\delta)$-differential privacy if:
\begin{equation}
    \Pr[\mathcal{M}(\mathcal{D}) \in \mathcal{R}] \leq e^\varepsilon \cdot \Pr[\mathcal{M}(\mathcal{D}') \in \mathcal{R}] + \delta
\end{equation}
where $\epsilon >0$ is the privacy budget controlling the level of indistinguishability between neighboring datasets, and $\delta \in [0,1]$ represents the upper bound of the allowable exception probability. For a given $\delta$, a larger $\epsilon$ value in the privacy-preserving mechanism leads to stronger distinguishability between neighboring datasets, thereby increasing the risk of privacy breaches. 

According to \cite{dwork2014algorithmic}, we present the DP mechanism by adding artificial Gaussian noise. \cite{9069945} has derived has derived that the noise distribution $n \sim \mathcal{N}(0,\sigma^2) $ satisfies $(\epsilon,\delta)$-DP for the uplink and downlink channels in the traditional FL framework. \cite{shi2021hfl} proposes a DP method under the HFL architecture, where DP is adopted in the edge aggregation and global aggregation. In our HFL-SNM framework, we mainly consider the data leakage between clients and ESs including uplink and downlink. To ensure $(\epsilon,\delta)$-DP for cilents in the uplink in one exposure, the noise scale is set $\sigma = \frac{c\Delta s_\mathrm{up}}{\epsilon}$. Accordingly, we present the uplink noise scale as follows:
\begin{align}
\sigma_\mathrm{up}^n = \frac{cC^n \Delta s_{up}^{\mathcal{D}_n}}{\epsilon}, 
\end{align}
where $\Delta s_\mathrm{up}^{\mathcal{D}_n}$ denotes the sensitivity of $\mathcal{D}_n$ in the edge aggregation. $C^n$ represent the number of exposure for local model of $u_n$. $c=\sqrt{2\ln{\frac{1.25}{\delta}}}$ is the scaling coefficient of the Gaussian mechanism, uniquely determined by the failure probability $\delta$.

Let $W$ denote the clipping threshold for bounding model parameters.
\begin{lemma}
The sensitivity for the edge aggregation is given by
\begin{align}
\Delta s_\mathrm{up}^{\mathcal{D}_n} = \frac{2W}{D_n},
\end{align}
\end{lemma}
\begin{proof}
The detailed proof is provided in Appendix B. 
\end{proof}
We use $Q$ to substitute the following parameters to simplify the computation.
\begin{align*}
   Q = (\frac{C^k}{D^{\mathcal{S}^k}})^2-\frac{1}{{m_k}^2} \sum_{n=1}^{m_k} (\frac{C^n}{D_n})^2,
\end{align*}
where $m$ is the number of clients in $\mathcal{S}^k$.
\begin{lemma}
The downlink noise scale is given by
\begin{align}
\sigma_\mathrm{down}^k = 
\begin{cases}
\frac{2cW\sqrt{Q}}{\epsilon} & \mathrm{if} \phantom{k}Q>0 \\
\phantom{kk}0 & \mathrm{otherwise}
\end{cases}
\end{align}
\end{lemma}
\begin{proof}
The detailed proof is provided in Appendix C. 
\end{proof}

\section{Problem Solution}
In this section, we will solve the problem analytically for the proposed objective function. Clearly, the problem is a mixed integer nonlinear program, and its combinatorial optimisation part falls under the category of NP-hard problems.Our objective is to select a suitable clients set $\mathcal{S}$ from $\mathcal{U}$ and allocate the clients to $\mathcal{K}$. The solution process involves considering multiple variables simultaneously, including the number of scheduled clients, the effective data, the redundant data, and the spatial relationship between clients and ESs. 

To tackle the problem, we decompose it into three sub-problems. We define the first sub-problem as \textbf{Resource allocation}, denoted as \textbf{P1}, where the optimal bandwidth and CPU frequency need to be solved for each client and ES. Based on the solution of \textbf{P1}, we propose \textbf{Edge association} as second sub-problem, denoted as \textbf{P2}, to obtain the optimal association solution. Finally, based on the \textbf{P1} and \textbf{P2}, we introduce \textbf{Client selection} as third sub-problem, denoted as \textbf{P3}, which is equal to the original problem to minimize the total energy consumption.
\subsection{Resource allocation}
First, we consider \textbf{P1} for a single ES. To simplify the objective function expression, we use $X_n$ and $Y_n$ to represent certain parameters, as detailed below.
\begin{align*}
X_{n}&=\lambda_n C{D}_n,\\
Y_{n}&=\frac{z}{\log_{2}\left(1+\frac{h_{n,k}p_n}{B_{n,k}N_{0}}\right)}.
\end{align*}

According to \textbf{P}, the constraint (15d) is a constraint on latency, we convert it to a bandwidth and frequency representation. According to (5) and (10), we derive the
constraint $\frac{X_n}{\nu_n}+\frac{Y_n}{B_{n,k}}\leq t_0$. Then we obtain \textbf{P1} for minimizing the overhead of a single ES $e_k$ under one global round as follows 
\begin{subequations}\label{eq:(16)}
\begin{align}
 \textbf{ P1:}\phantom{E} &\min_{\{ B_{n,k},\nu_n\}} E_k  = \tau\sum_{n \in \{i| u_i \in \mathcal{S}^k\}}\left(\frac{\alpha}{2} X_n\nu_{n}^{2}+\frac{Y_n p_n}{B_{n,k}}\right) \notag\\
\text{s.t.} \quad &\frac{X_n}{\nu_n}+\frac{Y_n}{B_{n,k}}\leq t_0, \quad \forall u_n \in \mathcal{S}, \thinspace e_k \in \mathcal{K},\\
& \nu_\mathrm{min}\leq\nu_{n}\leq\nu_\mathrm{max}, \quad \forall u_n \in \mathcal{S}, \\
& \sum_{n \in \{i| u_i \in \mathcal{S}^k\}} B_{n,k} \leq B_0, \quad \forall e_k \in \mathcal{K}, \mathcal{S}^k \subseteq \mathcal{S}, \\
& 0 \leq B_{n,k} \leq B_0, \quad \forall u_n \in \mathcal{S}^k, \thinspace e_k \in \mathcal{K}, \\
& k \in \{i| e_i \in \mathcal{K}\}, \mathcal{S}^k \subseteq \mathcal{S} \subseteq \mathcal{U}. 
\end{align}
\begin{lemma}
\textbf{P1} is convex.
\end{lemma}
\begin{proof}
The detailed proof is provided in Appendix D.
\end{proof}
\end{subequations}
To tackle the convex optimization problem, we derive the Karush-Kuhn-Tucker (KKT) conditions with Lagrange multiplier approach. The Lagrange formulation for \textbf{P1} is given by
{\footnotesize
\begin{align}\label{eq:(17)}
  L_k &= \tau\sum_{n \in \{i| u_i \in \mathcal{S}^k\}}(\frac{\alpha}{2}X_n\nu_n^2+\frac{Y_np_n}{B_{n,k}})+\mu_k\left(\sum_{n \in \{i| u_i \in \mathcal{S}^k\}}B_{n,k}-B_0\right) \notag \\
 &\phantom{\tau}+\sum_{n \in \{i| u_i \in \mathcal{S}^k\}}\gamma_n(\nu_\mathrm{min}-\nu_n)+\sum_{n \in \{i| u_i \in \mathcal{S}^k\}}\sigma_n(\nu_n-\nu_\mathrm{max}) \notag \\
 &\phantom{\tau}+\sum_{n \in \{i| u_i \in \mathcal{S}^k\}}\theta_n(\frac{X_n}{\nu_n}+\frac{Y_n}{B_{n,k}}-t_0),
\end{align}
}
where $\mu_k\geq 0,\theta_n\geq 0,\sigma_n\geq 0, \gamma_n\geq 0$ are Lagrange multipliers.
\begin{lemma}
the KKT conditions are described as 
\begin{subequations}
\begin{align}\label{eq:(18)}
\theta_n&=\tau\alpha \nu_{n}^3+\frac{(\sigma_n-\gamma_n)\nu_n^2}{X_n}, \\
\mu_k&=\frac{\tau p_n + \theta_n}{z\ln2}(\frac{Y_n}{B_{n,k}})^2(\ln (1+\frac{h_{n,k}p_n}{B_{n,k}N_0}) - \frac{1}{1+\frac{B_{n,k}N_0}{h_{n,k}p_n}}), \\
B_0&=\sum_{n \in \{i| u_i \in \mathcal{S}^k\}}B_{n,k}.
\end{align}
\end{subequations}
\end{lemma}
\begin{proof}
The detailed proof is provided in Appendix E. 
\end{proof}
 
\begin{lemma}\label{lemma2}
If the constraint (20a) on client $u_n$ is slack, then $\nu_n^* = \nu_{min}$.
\end{lemma}
\begin{proof}
The detailed proof is provided in Appendix F. 
\end{proof}

\begin{lemma}\label{lemma3}
Let $\Gamma_n = \frac{X_n}{\nu_n}+\frac{Y_n}{B_{n,k}}-t_0$.
 Then $\Gamma_n$ is monotonically decreasing with respect to $\nu_n$ while increasing with respect to $\nu_i, i \neq n$.
\end{lemma}
\begin{proof}
The detailed proof is provided in Appendix G. 
\end{proof}
Combining the KKT conditions and the aforementioned Lemmas, the AO algorithm is proposed to solve the closed-form solution of \textbf{P1}. The AO algorithm is divided into four parts, described as follows.
\begin{algorithm}
\DontPrintSemicolon
\KwIn{Related parameters of clients in $\mathcal{S}^k$. $h_{i,k},p,\alpha,C,N_0$.}
\KwOut{Optimal Frequency $\nu_i^*$, Optimal Bandwidth $B_{i,k}^*$}
Initialize the CPU frequency $\nu_i = \nu_\mathrm{min}$ for $u_i$ in $\mathcal{S}^k$\;
Initialize slack constraint clients set $\mathcal{L}^k $, binding constraint clients set $\mathcal{T}^k$\;
\While{$\mathcal{L}^k < |\mathcal{S}^k|$}{
    \For{$u_i$ in $\mathcal{T}^k$}{
        calculate $\nu_i$ and $B_{i,k}$ under KKT conditions \;
        $\mathcal{L}^k + \{u_i\}$, $\mathcal{T}^k - \{u_i\}$
    }
    \For{$u_j$ in $\mathcal{S}^k$}{
        Calculate $\Gamma_n$\;
        \If{$\Gamma_n > 0$}{
            $\mathcal{L}^k - \{u_j\}$\;
            $\mathcal{T}^k + \{u_j\}$
        }
    }
}
$\nu_i^* =  \nu_i$, $B_{i,k}^* =B_{i,k}$ for $u_i$ in $\mathcal{S}^k$\;
\Return{$\nu_i^*,B_{i,k}^*$}\;
\caption{Alternating Optimization}
\end{algorithm}
\subsubsection{Step1}
We denote the set of clients with slack constraint (20a) as as $\mathcal{L}^k$. We assume the constraint (20a) is slack for all clients in the beginning. Then we have $\nu_n =\nu_\mathrm{min}$, $ n \in \{i| u_i \in \mathcal{S}^k\}$.
\subsubsection{Step2}
Substitute $\nu_n$ into $\Gamma_n$ and determine the sign of $\Gamma_n$.
\begin{lemma}
After Step2, if $\Gamma_n > 0$, the constraint (19a) on $u_n$ is binding.
\end{lemma}
\begin{proof}
The detailed proof is provided in Appendix H. 
\end{proof}
\subsubsection{Step3}
We define the clients with binding constraint (20a) as $\mathcal{T}^k$. Based on the lemma4, we obtain $\mathcal{T}^k$. According to $\mathcal{L}^k$ and KKT conditions, we derive the necessary condition of the $\nu_i^*$ and $B_{i,k}^*$ for $u_i$ in $\mathcal{S}^k$. 
\subsubsection{Step4}
Substituting the derived  necessary condition of $\nu_i^*$ and $B_{i,k}^*$ into $\Gamma_n$, repeat step2 and step3 until there is no client that satisfies $\Gamma_n > 0$. Therefore, we can derive the constraint conditions of all clients and obtain the optimal frequency allocation $\nu_n^*$ and $B_{n,k}^*$ for all clients.

Substituting the optimal parameter into the objective function of \textbf{P1}, we get the minimal energy consumption of $e_k$ as 
\begin{equation}\label{eq:(19)}
E_k^* =\tau\sum_{n \in \{i| u_i \in \mathcal{S}^k\}}\left(\frac{\alpha}{2} X_n{\nu_{n}^*}^{2}+\frac{Y_np_n}{B_{n,k}^*}\right).
\end{equation}
Once the clients scheduled by each ES are determined, the AO algorithm is utilized to obtain the closed-form solution under the constraints of \textbf{P1}. Compared to traditional algorithms, the AO algorithm has a lower complexity. When traversing the relaxed states of the constraint (20a) for all clients, the time complexity is $\mathcal{O}(2^m)$, where $m$ is the number of clients scheduled by a single ES. While in the AO algorithm, each client only needs to evaluate the relaxed state of constraint (20a) once, so the maximum time complexity is $\mathcal{O}(m)$. Therefore, the AO algorithm is highly efficient.

\subsection{Edge association}
Based on \textbf{P1}, we proceed to address the optimization solution of \textbf{P2}. By referencing \textbf{P} and incorporating the results of \textbf{P1}, we derive \textbf{P2} as follows.
\begin{subequations}\label{eq:(21)}
\begin{align}
 \textbf{ P2:}\phantom{E_k} &\min_{\{\mathcal{S}^k\}} \phantom{a}E_\mathrm{total}  =  \sum_{k \in \{i| \mathcal{S}^i \subseteq \mathcal{S}\}}E_k^*  \notag\\
\text{s.t.} \quad &\sum_{n \in \{i| u_i \in \mathcal{S}^k\}}\delta_{n,k} = 1, \quad \forall \mathcal{S}^k \subseteq \mathcal{S}, \thinspace e_k \in \mathcal{K}, \\
& \bigcup_{e_k \in \mathcal{K}} \mathcal{S}^k = \mathcal{S}, \quad \forall \mathcal{S} \subseteq \mathcal{U}.
\end{align}
\end{subequations}
\textbf{Edge association} is a classic optimization problem in the field of wireless communications, which is challenging to solve due to its NP-hard nature. While traditional heuristic algorithms \cite{pearl1984heuristics,holland1992adaptation,kirkpatrick1983optimization,glover1989tabu} suffer from prohibitive computational complexity when handling large-scale client access. To address this limitation, we propose a Fast Greedy algorithm based on the AO algorithm. The specific steps are as follows.
\begin{itemize}
    \item \textbf{Data Sorting} We sort the clients in $\mathcal{S}$ by $D_n$ in descending order, yielding $\mathcal{S}'$.
    \item \textbf{Sever Selection} Traverse the clients in $\mathcal{S}'$ and associate each client with the ESs that currently minimizes $E_\mathrm{total}$.
\end{itemize}
\begin{algorithm}
\DontPrintSemicolon
\KwIn{The cilent selection $\mathcal{S}$.}
\KwOut{The optimal association scheme ${\mathcal{S}^k}^*$.}
Initialize the total energy consumption $E_\mathrm{total} = 0$\;
$\mathcal{S}'$ = Sorted$(\mathcal{S}$, $\mathrm{key}= D_n)$\;
\For{$u_n $ in $ \mathcal{S}' $}
{
    $\Delta E = \infty $\;
    \For{$e_k$ in $\mathcal{K}$}
    {
        calculate $E_\mathrm{total}'$ by AO\;
        $\Delta E = \min \{\Delta E,E_\mathrm{total}'-E_\mathrm{total}\}$\;
    }
    $E_\mathrm{total} = E_\mathrm{total}+\Delta E$\;
    $k=\underset{{k'}\in \{i|e_i \in \mathcal{K}\}}{{\arg\min} \, } E_\mathrm{total}'-E_\mathrm{total}$\;
    ${\mathcal{S}^k}^*+\{u_n\}$
}
\Return{${\mathcal{S}^k}^*$}\;
\caption{Fast Greedy Algorithm}
\end{algorithm}
The proposed algorithm is mathematically described below:
{\footnotesize
\begin{equation}
\delta_{n,k} = 
\begin{cases}
1, &  \mathrm{if} \phantom{x} k= \underset{{k'}\in \{i|e_i \in \mathcal{K}\}}{{\arg\min} \, } E_{k'}^*(\mathcal{S}^{k'} \bigcup \{u_n\}) + \sum_{j \neq {k'}} E_j^*(\mathcal{S}^j) , \\
0, & \mathrm{otherwise},
\end{cases}
\end{equation}}

Although traversing all solutions of \textbf{P2} can find the optimal solution, the time complexity reaches $\mathcal{O}(K^M)$. In contrast, the FG algorithm offers a more efficient alternative with a significantly reduced time complexity of $\mathcal{O}(M^2)$. This makes the FG algorithm much more scalable and practical for real-time applications, while still providing results that are close to the optimal solution.

\subsection{Client selection}
Based on the solutions of \textbf{P1} and \textbf{P2}, \textbf{P3} is proposed to select clients from the social network to minimize total energy consumption, which is equivalent to \textbf{P}. According to (23) and (24), \textbf{P3} is characterized as 
\begin{subequations}\label{eq:(22)}
\begin{align}
 \textbf{ P3:}\phantom{E_k} &\min_{\{\mathcal{S}\}} \phantom{k}E_\mathrm{total}^* =  \sum_{k \in \{i| {\mathcal{S}^i}^* \subseteq \mathcal{S}\}}\tau\sum_{n \in \{j| u_j \in \mathcal{S}^k\}}E_k^* \phantom{E_kE_k} \notag \\
\text{s.t.} \quad &\mathcal{S} \subseteq \mathcal{U} \\
& r_\mathrm{ef} \geq r_\mathrm{ef}^0 
\end{align}
\end{subequations}
The $E_\mathrm{total}^*$ represents the minimal energy consumption in one global round, which is calculated by optimal frequency $\nu_n^*$ and bandwidth allocation $B_{n,k}^*$ in \textbf{P1} and optimal clients allocation ${\mathcal{S}^k}^*$ in \textbf{P2}. Let $\mathcal{B}_\mathcal{S}$ denote the upper bound of $E_\mathrm{total}^*$, which is derived as

{\footnotesize
\begin{align*}
E_\mathrm{total}^* &= \sum_{ k \in \{i| \mathcal{S}^i \subset \mathcal{S}\}}\tau\sum_{n \in \{j| u_j \in \mathcal{S}^k\}}(\frac{\alpha}{2}\lambda_n C (\nu_n^*)^2 {D}_n + p_n t_{n,k}^\mathrm{com} )  \\
& \leq  \sum_{k \in \{i| \mathcal{S}^i \subset \mathcal{S}\}}\tau\sum_{n \in \{j| u_j \in \mathcal{S}^k\}}(\frac{\alpha}{2}\lambda_n C (\nu_n^*)^2 {D}_n + p_n(t_0-t_{n,k}^\mathrm{cmp}) ) \notag \\
& =  \sum_{ k \in \{i| \mathcal{S}^i \subset \mathcal{S}\}}\tau\sum_{n \in \{j| u_j \in \mathcal{S}^k\}}(\frac{\alpha}{2}\lambda_n C (\nu_n^*)^2 {D}_n  + p_n(t_0-\frac{\lambda_n C D_n}{\nu_n^*}) ) \notag \\
& \leq  \sum_{ k \in \{i| \mathcal{S}^i \subset \mathcal{S}\}}\tau\sum_{n \in \{j| u_j \in \mathcal{S}^k\}}(\frac{\alpha}{2}\lambda_n C \nu_\mathrm{max}^2{D}_n+p_n(t_0-\frac{\lambda_n C D_n}{\nu_\mathrm{max}})) \notag \\
& = \tau\sum_{n \in \{i| u_i \in \mathcal{S}\}}(\frac{\alpha}{2}\lambda_n C \nu_\mathrm{max}^2{D}_n+p_n(t_0-\frac{\lambda_n C D_n}{\nu_\mathrm{max}})) \notag \\
& = \mathcal{B}_\mathcal{S}.
\end{align*}}

We have proved that the model performance is affected by effective data and redundant data. Therefore, we introduce $\epsilon_\mathcal{S} = \frac{{D}_{\mathrm{ef}}^{\mathcal{S}}}{{D}_{\mathrm{re}}^{\mathcal{S}}}$  as a metric to represent the model performance. A larger $\epsilon_\mathcal{S}$ indicates better model performance.

From the perspective of the number of clients, if the number of scheduled clients exceeds a certain threshold, it becomes infeasible to ensure that all clients satisfy constraint (19a), resulting in the infeasibility of \textbf{P}. To address this, we use an averaging method to estimate the upper bound on the number of scheduled clients. Let $\bar{\mathcal{H}}$ denote the upper bound on the number of clients scheduled by a single ES and $\mathcal{H} = \lfloor K\bar{\mathcal{H}}\rfloor$ as the total upper bound, which can be solved by
\begin{equation}\label{eq:(23)}
    \frac{\lambda_n C \bar{{
D}}}{\bar{\nu}} +\frac{z}{\bar{B}\ln \left(1+\frac{\bar{h} \bar{p}}{\bar{B} N_0} \right)} = t_0,
\end{equation}
where $\bar{{
D}} = \frac{\sum_{n \in \{j| u_j \in \mathcal{U}\}}{{
D}}_{n}}{|\mathcal{U}|}$, $\bar{B} = \frac{B_0}{\bar{\mathcal{H}}}$. $\bar{h}$ is calculated based on the average distance of all clients from their communicable ES. And $C$, $\bar{\nu}$, $\bar{p}$ are averages taken from a settable range. 

\begin{algorithm}
\DontPrintSemicolon
\KwIn{The lower bound and upper bound of scheduled clients, the EDCR constraint $\mathcal{L}, \mathcal{H}$, $r_\mathrm{ef}^{0}$.}
\KwOut{The optimal client selection $\mathcal{S}^*$ }
Initialize the Remaining Clients $\mathcal{R} = \mathcal{U}$, $G_\mathrm{min}=\infty$, $\mathcal{P} = None$\;

\For{$M = \mathcal{L}$ \textbf{to} $\mathcal{H}$}{
    \For{$x = 1$ \textbf{to} $\xi$}{
        $M' = M $, $\mathcal{R} = \mathcal{U}$, $\mathcal{S}=None$\;
        \While{$M' > 0 $}{
    Initialize $\bar{{D}_\mathrm{ef}} = \frac{{D}_\mathrm{ef}^{\mathcal{R}}r_\mathrm{ef}^{0}}{M_r}$, $\mathcal{F} = None$\;
    \For{$u_n$ in $\mathcal{R}$}{
        \If{${D}_\mathrm{ef}^{\{u_n\}}\geq \bar{{D}_\mathrm{ef}}$}{
        $\mathcal{F}$ + $u_n$ \;
         }
    }
    \If{$\mathcal{F} = None$}{
        $u^* = {\{u_n:{D}_\mathrm{ef}^{\{u_n\}} = \max_{u_i \in \mathcal{R}}{D}_\mathrm{ef}^{\{u_i\}}\}}$ \;
    }
    \Else{
        Randomly select a client from $\mathcal{FC}$ as $u^*$\;
    }
    $\mathcal{S}$ + $\{u^*\}$,$\mathcal{R}$ - $\{u^*\}$,$M' - 1$
}
    $\mathcal{P}$ + $\{\mathcal{S}\}$
    }
}
\For{$\mathcal{S}$ in $\mathcal{P}$}{
    \If{$G_\mathcal{S} \leq G_\mathrm{min}$}{
    $G_\mathrm{min} = G_\mathcal{S}$\;
    $\mathcal{S}^*= \mathcal{S}$
    }
}
\Return{$\mathcal{S}^*$}\;
\caption{Performance-Energy Metric Optimization}
\end{algorithm}
According to the upper bound $\mathcal{H}$, we can obtain the maximum EDCR through Greedy Algorithm. Thus we set the EDCR constraint $ r_\mathrm{ef}^{0} \leq r_\mathrm{ef}^\mathrm{max}$ to avoid the unsolvable condition as much as possible. Similarly, the lower bound of the schedulable number of clients is obtained through Greedy algorithm, denoted as $\mathcal{L}$. Then the actual number of scheduled number $M$ is in the range of $[{\mathcal{L}},{\mathcal{H}}]$.

We introduce a metric factor $G_\mathcal{S}$ to evaluate the quality of $\mathcal{S}$, which is defined as 
\begin{equation}\label{eq:27}
    G_\mathcal{S}=\frac{\epsilon_\mathrm{max}\mathcal{B}_\mathcal{S}}{\epsilon_\mathcal{S}\mathcal{B}_\mathrm{max}}.
    \end{equation}
$G_\mathcal{S}$ is positively correlated with $\mathcal{B}_\mathcal{S}$ and negatively correlated with $\epsilon_\mathcal{S}$. Thus a smaller $G_\mathcal{S}$ indicates better model performance and lower energy consumption. Based on the analysis, the PEMO algorithm is proposed, which is shown in \textbf{Algorithm 2}.  
\begin{itemize}
    \item \textbf{Scheme Generation}: Let M vary from $\mathcal{L}$ to $\mathcal{H}$. For each value of $M$, $\xi$ distinct selections $\mathcal{S}$ are generated from $\mathcal{U}$, each of which satisfy constraint(19a). Let $\mathcal{P}$ denote the set of all feasible selections. 
    \item \textbf{Metric Comparison}: Compare the values of $G_\mathcal{S}$ across various selections in $\mathcal{P}$. The selection with the smallest $G_\mathcal{S}$ is $\mathcal{S}^*$.
\end{itemize}
PEMO algorithm is also a heuristic algorithm similar to genetic algorithms, which selects the best solutions from multiple generations of candidate solutions. Its time complexity is $\mathcal{O}(\xi(\mathcal{H}-\mathcal{L}+1)M^2)$, while the time complexity of traversing all solutions is $\mathcal{O}(\frac{N!}{M!(N-M)!})$. In comparison, the PEMO algorithm significantly optimizes the complexity.

The sequential execution of PEMO, FG, and AO constitutes the DO-SNM algorithm. Since the PEMO algorithm is independent for P1 and P2, the overall time complexity of the DO-SNM algorithm is $\mathcal{O}((\xi(\mathcal{H}-\mathcal{L}+1)+1)M^2)$
The experiments in the following section have shown that DO-SNM provides an excellent solution, keeping energy consumption at a low level while ensuring the model performance.

\section{Experiment}
We consider 4 ESs in a square area $1000$m $\times$ $1000$m, and 60 clients with density 2 are randomly distributed in the coverage area of the 4 ESs. The detailed parameters setting is provided in Table I. 
\begin{table*}[h!]
\caption{Experimental parameters setting.}
\centering
\renewcommand{\arraystretch}{1.5} 
\begin{tabular}{|>{\centering\arraybackslash}m{2cm}|>{\centering\arraybackslash}m{4cm}|>{\centering\arraybackslash}m{2cm}|>{\centering\arraybackslash}m{4cm}|}
\hline
\textbf{Parameter} & \textbf{Value} & \textbf{Parameter} & \textbf{Value} \\ \hline
$\tau$ & 1 & $p_n$ & $0.1 \sim 1$W \\ \hline
 $\alpha$ & $2 \times 10^{-28}$ & $f_n$ & $1 \sim 4$GHz \\ \hline
$B_0$ & 10MHz& $N_0$ & $4 \times 10^{-21}$W/Hz \\ \hline
$\nu_{min}$ & 1GHz& $\nu_{max}$ & 10GHz \\ \hline
$\eta^\mathrm{CIFAR-10}$ & 0.0005  & $\eta^\mathrm{IEMOCAP}$ & 0.0002 \\ \hline
$C^\mathrm{CIFAR-10}$ & 90822 cycles/sample& $C^\mathrm{IEMOCAP}$ & 15137 cycles/sample \\ \hline
$t_0^\mathrm{CIFAR-10}$ & 0.2s& $t_0^\mathrm{IEMOCAP}$ & 0.1s \\ \hline

\end{tabular}
\label{tab:notations}
\end{table*}

\begin{figure*}[htbp] 
\centering
\subfloat[CIFAR-10]{
    \begin{minipage}[b]{.4\linewidth}
        \centering
        \includegraphics[scale=0.37]{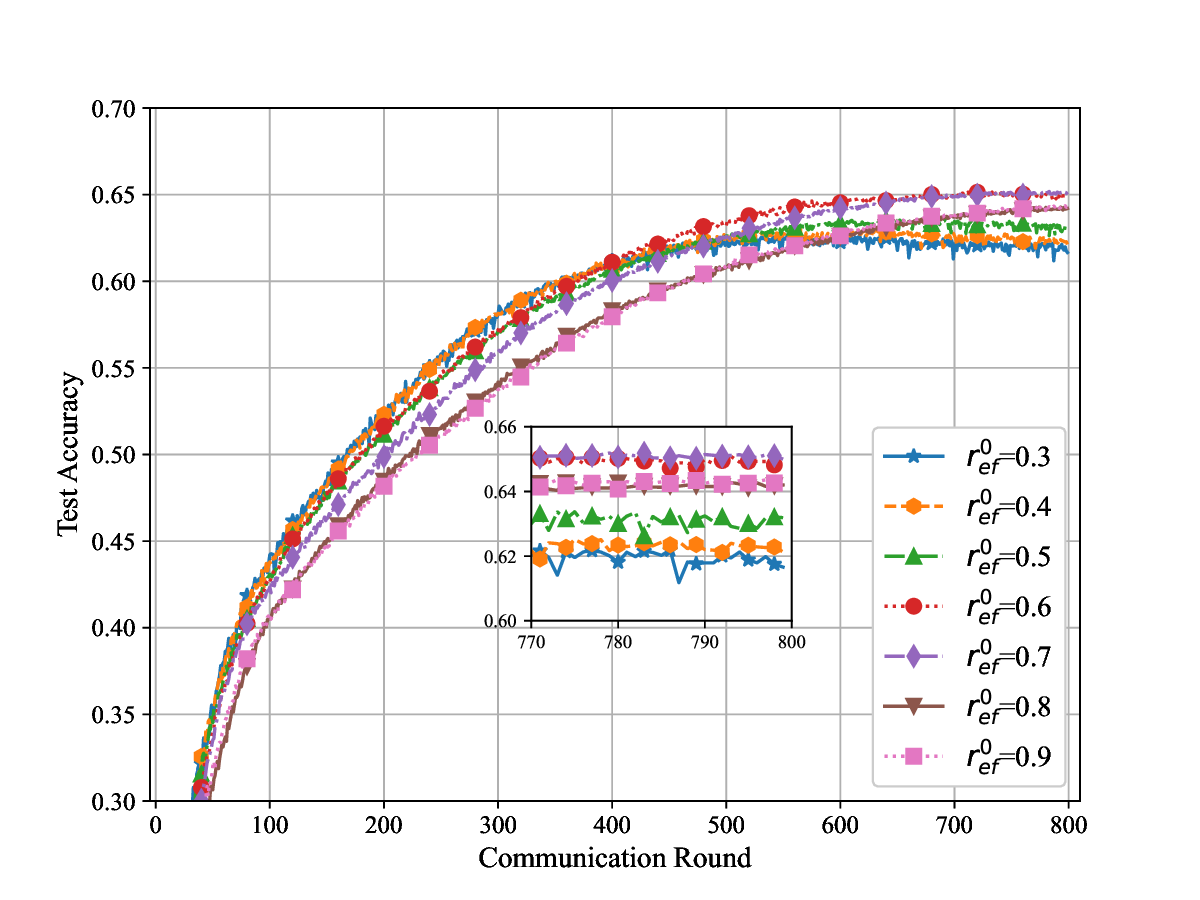}
        \end{minipage}
} 
\subfloat[IEMOCAP]{%
    \begin{minipage}[b]{.4\linewidth}
        \centering
        \includegraphics[scale=0.37]{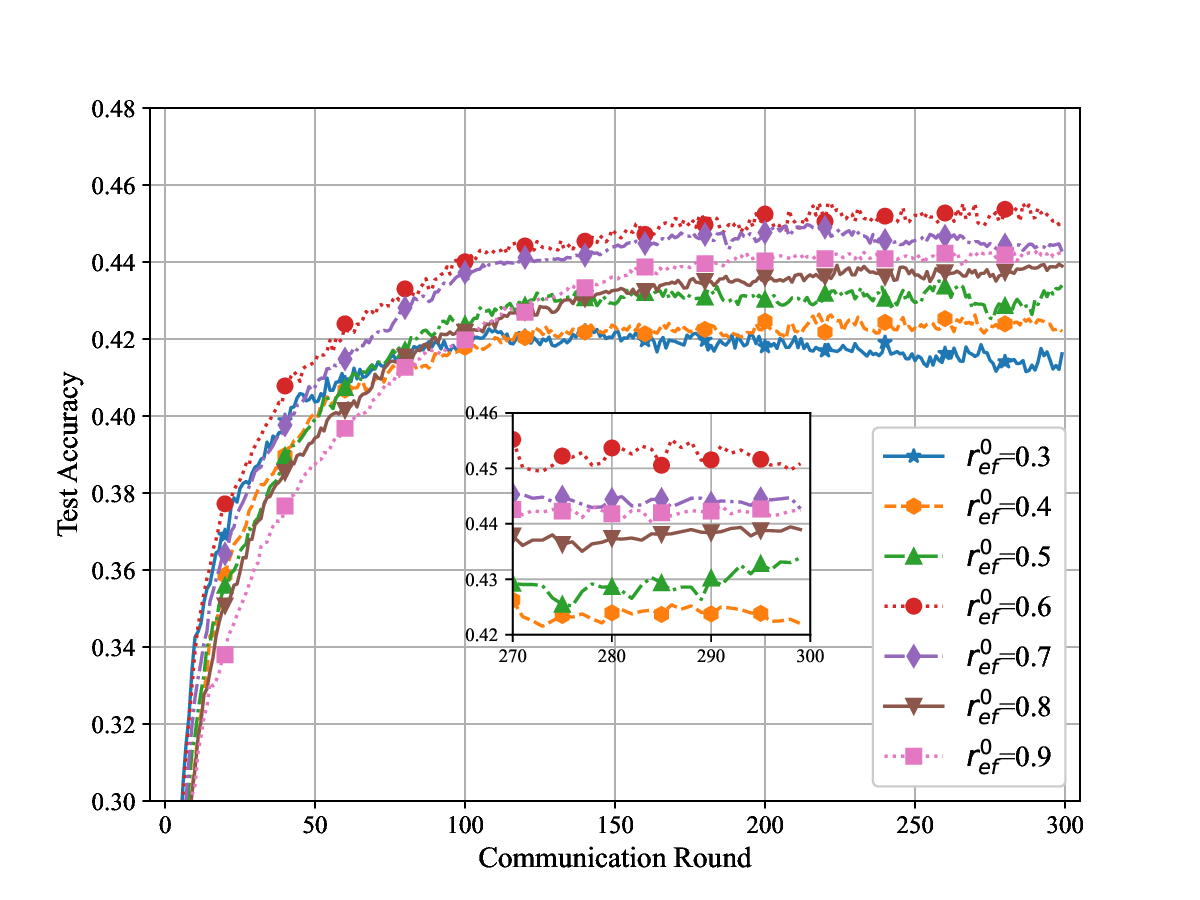}
        \end{minipage}
} \\
\subfloat[CIFAR-10]{%
    \begin{minipage}[b]{.4\linewidth}
        \centering
        \includegraphics[scale=0.37]{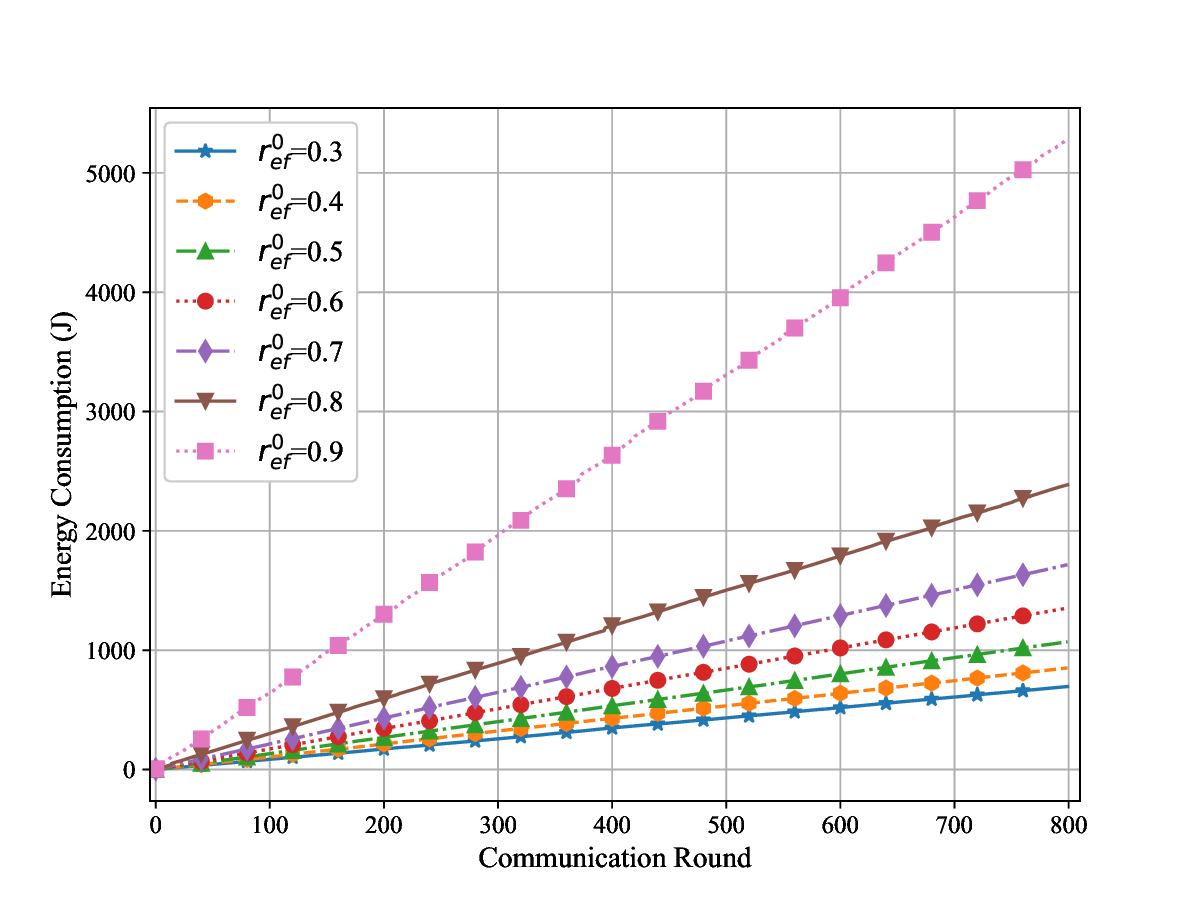}
        \end{minipage}
} 
\subfloat[IEMOCAP]{%
    \begin{minipage}[b]{.4\linewidth}
        \centering
        \includegraphics[scale=0.37]{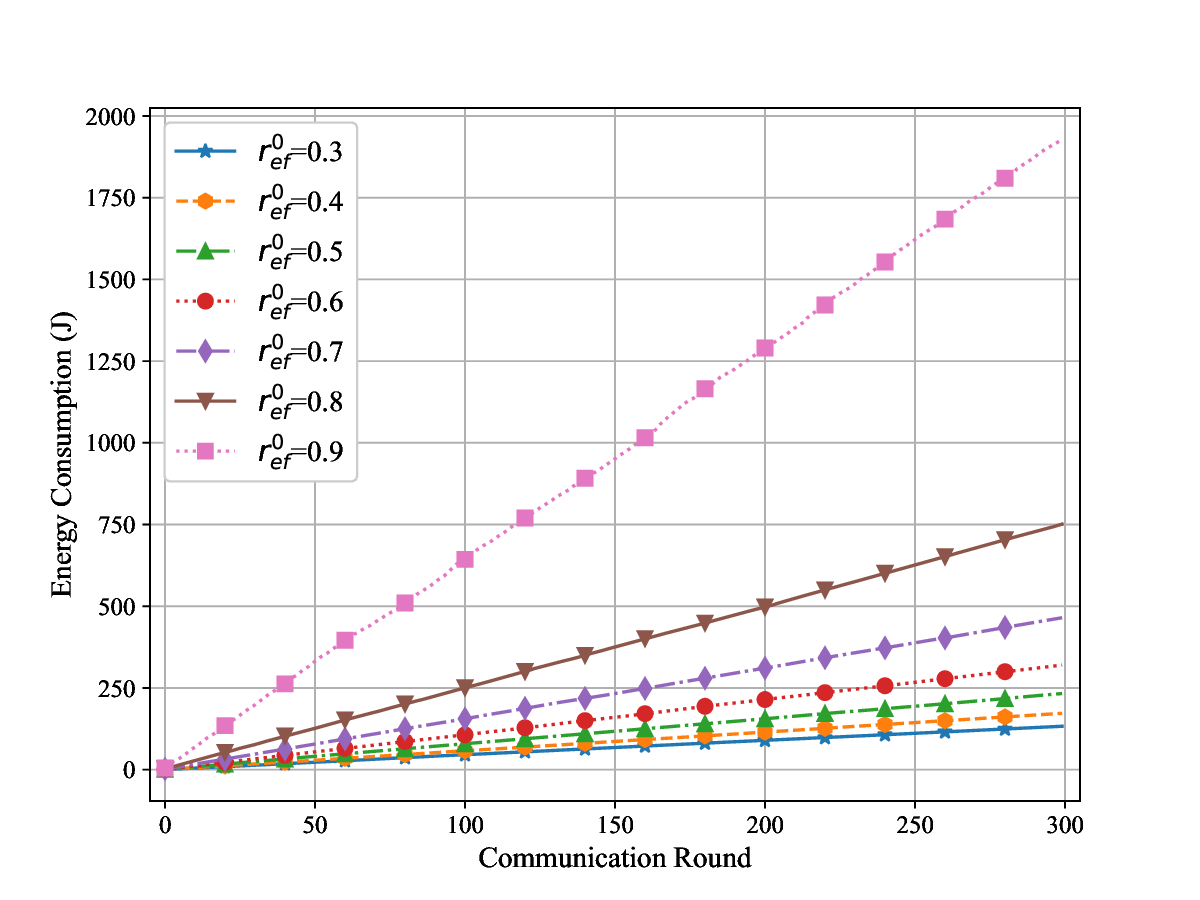}
        \end{minipage}
} 
\caption{Test accuracy and energy consumption of CIFAR-10 and IEMOCAP datasets under different $r_\mathrm{ef}^0$}
\end{figure*}
Considering that the social content of actual clients including information such as voice, images and text, we apply two datasets CIFAR-10 and IEMOCAP to evaluate our proposed algorithm. We employ the Dirichlet distribution to construct Non-IID datasets\cite{hsu1909measuring}. For each client, we generate a class distribution vector parameterized by $\rho$, then sample and allocate instances from each category according to the proportions specified in the vector. The value of $\rho$ determines the degree of distribution skewness - smaller $\rho$ values result in more heterogeneous data partitions. In this work, we set $\rho = 0.6$ to establish Non-IID data distributions for both the CIFAR-10 and IEMOCAP datasets.

The CIFAR-10 dataset consists of 60,000 color images of size $32$ $\times$ $32$ pixels, with 50,000 training images and 10,000 test images. Each image has 3 color channels (RGB). We employ a convolutional neural network (CNN) model to perform image classification on the CIFAR-10 dataset. Our model consists of three convolutional layers with $5$ $\times$ $5$ kernels (the first two with 32 filters, the last with 64 filters), each followed by a max-pooling layer with a $3$ $\times$ $3$ window. This is followed by one flatten layer and two fully connected layers (respectively with 64 units and 10 units). 

The IEMOCAP dataset is a widely used dataset for emotion recognition tasks, specifically designed to analyze emotional expressions in multimodal settings. The dataset includes both text, audio and video features, making it suitable for emotion recognition tasks in the audio-visual domain. In this experiment, we focus on the text features. Each sentence is preprocessed into a 100 dimensional vector.
We employ a recurrent neural network (RNN) model based on a Long Short-Term Memory (LSTM) architecture. Our model comprises two LSTM layers (with a hidden dimension of 50 units, the dropout rate of 0.1), one layer normalization, and one fully connected layer with 7 units. 
\begin{figure*}[htbp] 
\centering
\subfloat[CIFAR-10]{
    \begin{minipage}[b]{.4\linewidth}
        \centering
        \includegraphics[scale=0.37]{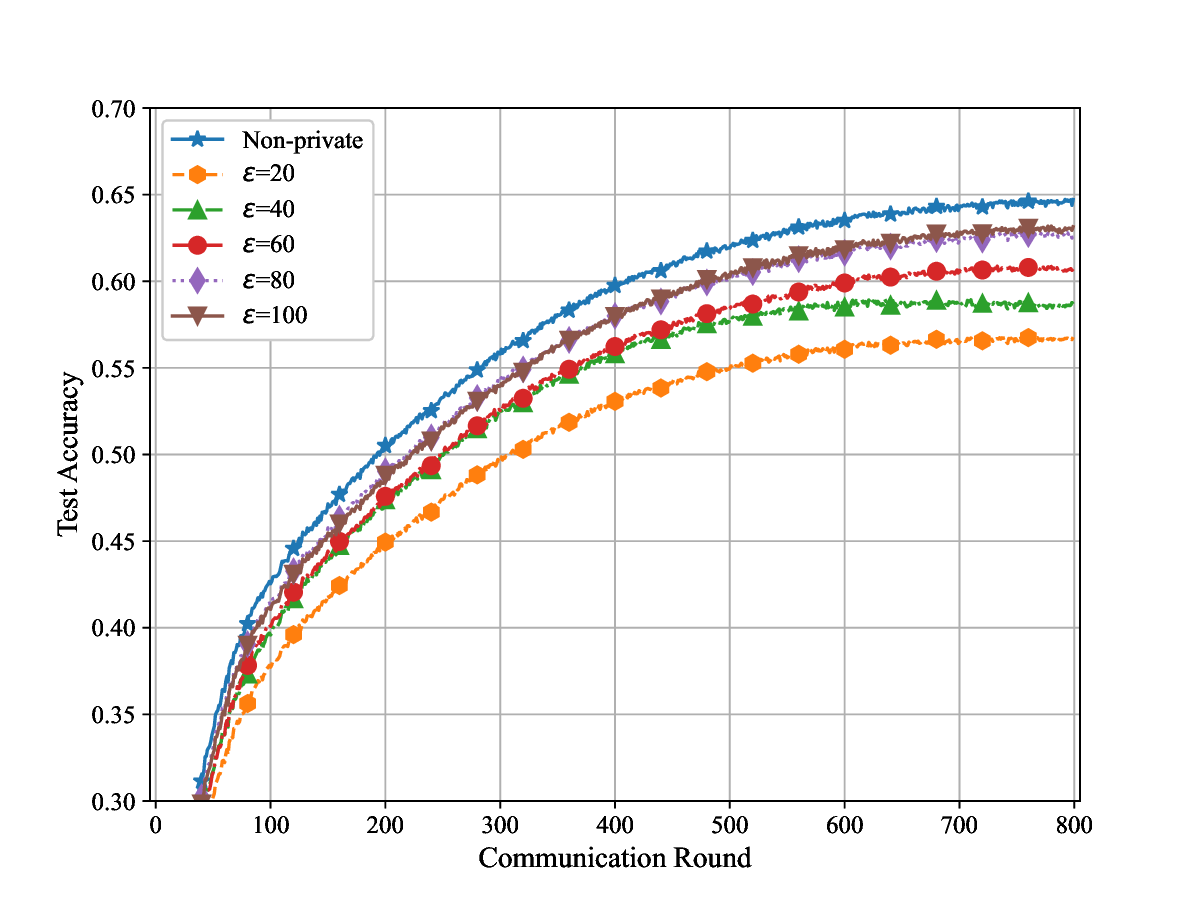}
        \end{minipage}
}
\subfloat[IEMOCAP]
{
    \begin{minipage}[b]{.4\linewidth}
        \centering
        \includegraphics[scale=0.37]{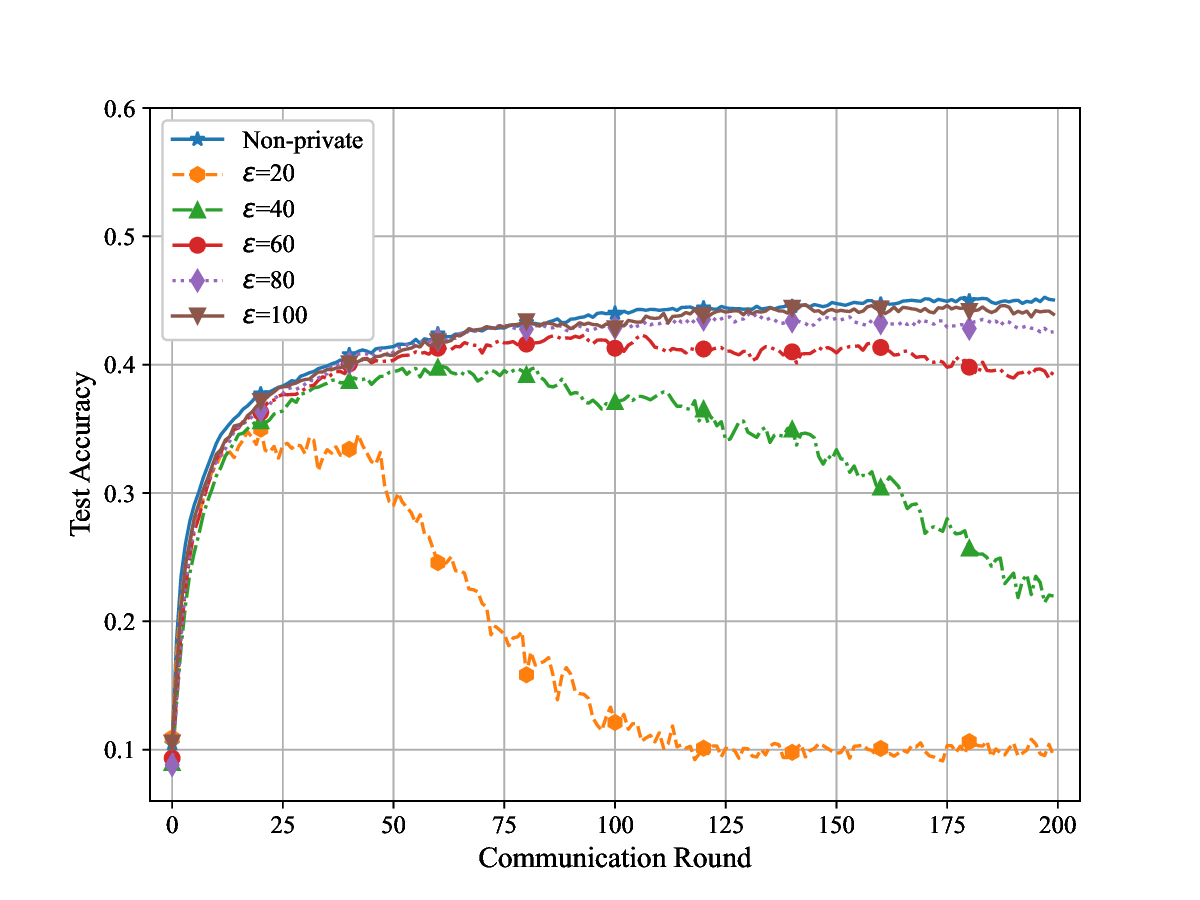}
        \end{minipage}
}
\caption{Test accuracy of CIFAR-10 and IEMOCAP datasets under different $\epsilon$.}
\end{figure*}
\subsection{Trade-off between Performance and Energy}
Fig.3 shows that the test accuracy and energy consumption of the DO-SNM algorithm on the CIFAR-10 and IEMOCAP datasets under varying $r_\mathrm{ef}^0$ values. We find that a small $r_\mathrm{ef}^0$ ($0.3, 0.4, 0.5$) results in lower energy consumption but leads to poor test accuracy. This is due to the poor effective data and Non-IID data distribution, which prevents the model from adequately learning the data across all clients. Large $r_\mathrm{ef}^0$ ($0.8, 0.9$) leads to increased energy consumption while providing a improvement in performance. This occurs because, as the number of scheduled clients increases, more effective data and redundant data is introduced into the training process. Consequently, the model learns adequate information from clients while repeatedly trains redundant data, which eventually causes overfitting. We observe that with $r_\mathrm{ef}^0$ values of 0.6 and 0.7, the communication overhead remains relatively moderate while achieving strong accuracy performance. This suggests that under this setting, the proposed algorithm strikes a trade-off between accuracy and energy consumption. From the perspective of energy, we choose $r_\mathrm{ef}^0 = 0.6$ as the optimal EDCR.

Therefore, setting $r_\mathrm{ef}^0$ too low compromises model performance, whereas excessively high values increase energy consumption and exacerbate overfitting, which reduces the model's generalization capability. Experimental results show that setting $r_\mathrm{ef}^0 = 0.6$ achieves a trade-off between test accuracy and energy overhead. In subsequent algorithm comparisons, $r_\mathrm{ef}^0$ is consistently set to 0.6 to ensure fair evaluation conditions.

\subsection{Evaluation on Protection Levels}
Fig. 4 demonstrates the impact of differential privacy guarantees on model performance by evaluating test accuracy under varying protection levels ($\epsilon$ ranging from 40 to 100) in our DO-SNM algorithm. For benchmarking purposes, a non-private baseline is included. The results exhibit a characteristic privacy-accuracy tradeoff: test accuracy monotonically decreases as privacy protection strengthens (with smaller $\epsilon$, the curve exhibit overfitting behavior). Based on this observation, we select $\epsilon=80$ as the operating point for subsequent experiments, to achieve a balanced compromise between privacy preservation and model accuracy.

\subsection{Performance of DO-SNM}
We evaluated DO-SNM against the following four baseline algorithms on CIFAR-10 and IEMOCAP datasets under the configuration of $r_\mathrm{ef}^0 = 0.6$ and $\epsilon = 80$. As shown in Fig. 5, DO-SNM demonstrated significantly faster convergence than the baselines, achieving test accuracies of $63\%$ and $44\%$ respectively while consuming substantially less energy.
\begin{itemize}
    \item[\textbullet] \textbf{Random Allocation (RA)\cite{huang2020efficiency}}:
    Randomly select clients from $\mathcal{U}$, then randomly assign them to $K$ ES. Bandwidth is evenly distributed among the clients, and frequency is randomly allocated within the feasible range.
    \item[\textbullet] \textbf{Location Greedy (LG)\cite{luo2020hfel}}: 
    Randomly select clients from $\mathcal{U}$, then each client chooses the nearest ES to upload model. Use \textbf{Algorithm 1} to allocate resources.
    \item[\textbullet] \textbf{Redundancy Driven (RD)}:
    Select clients from $\mathcal{U}$ to reach the least redundant data, then utilize FG algorithm to associate the ES with the clients and apply \textbf{Algorithm 1} to allocate resources.
    \item[\textbullet] \textbf{Effectiveness  Driven (ED)}: Select clients from $\mathcal{U}$, ensuring that the clients cover the maximum amount of effective data. Then, utilize FG algorithm to associate the ES with the clients and apply \textbf{Algorithm 1} to allocate resources.
\end{itemize}

\begin{figure*}[htbp]
\centering
\subfloat[CIFAR-10]
{
    \begin{minipage}[b]{.4\linewidth}
        \centering
        \includegraphics[scale=0.37]{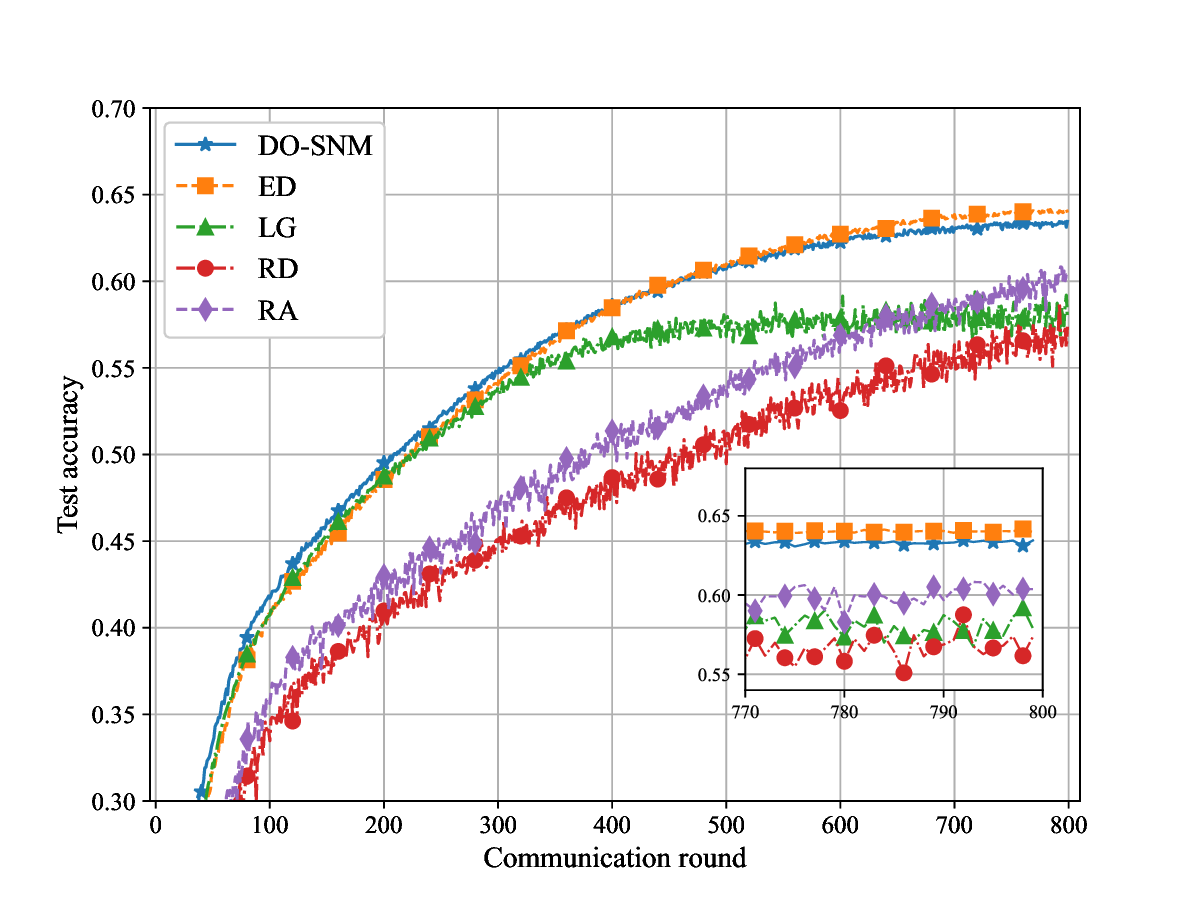}
        \end{minipage}
}
\subfloat[IEMOCAP]
{
    \begin{minipage}[b]{.4\linewidth}
        \centering
        \includegraphics[scale=0.37]{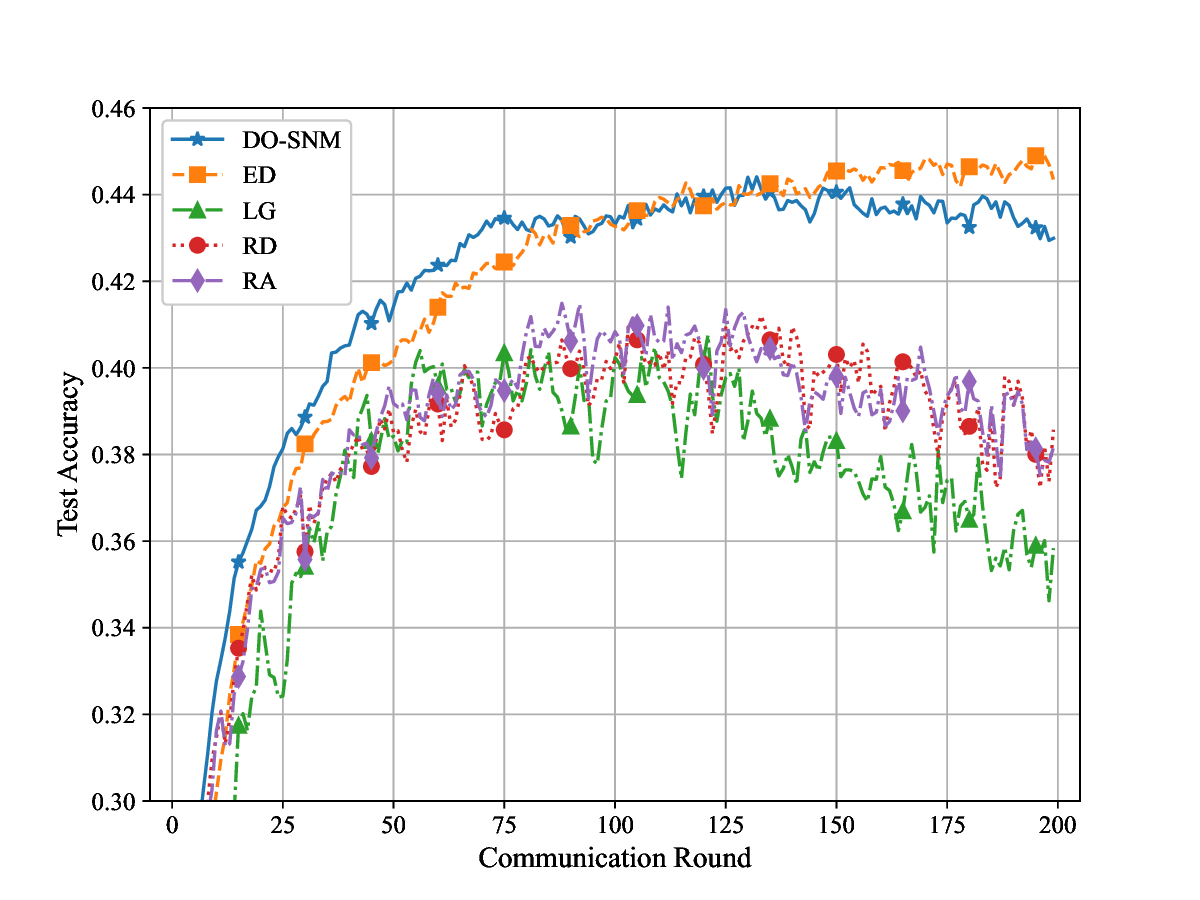}
        \end{minipage}
}\\
\subfloat[CIFAR-10]
{
    \begin{minipage}[b]{.4\linewidth}
        \centering
        \includegraphics[scale=0.37]{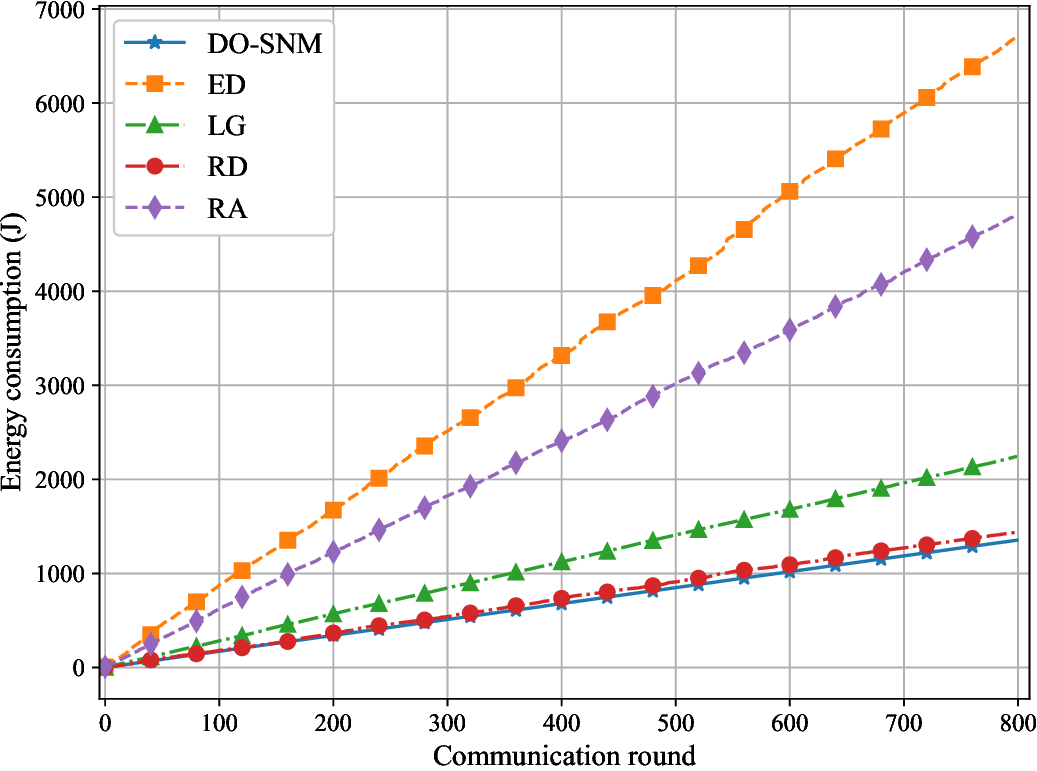}
        \end{minipage}
}
\subfloat[IEMOCAP]
{
    \begin{minipage}[b]{.4\linewidth}
        \centering
        \includegraphics[scale=0.37]{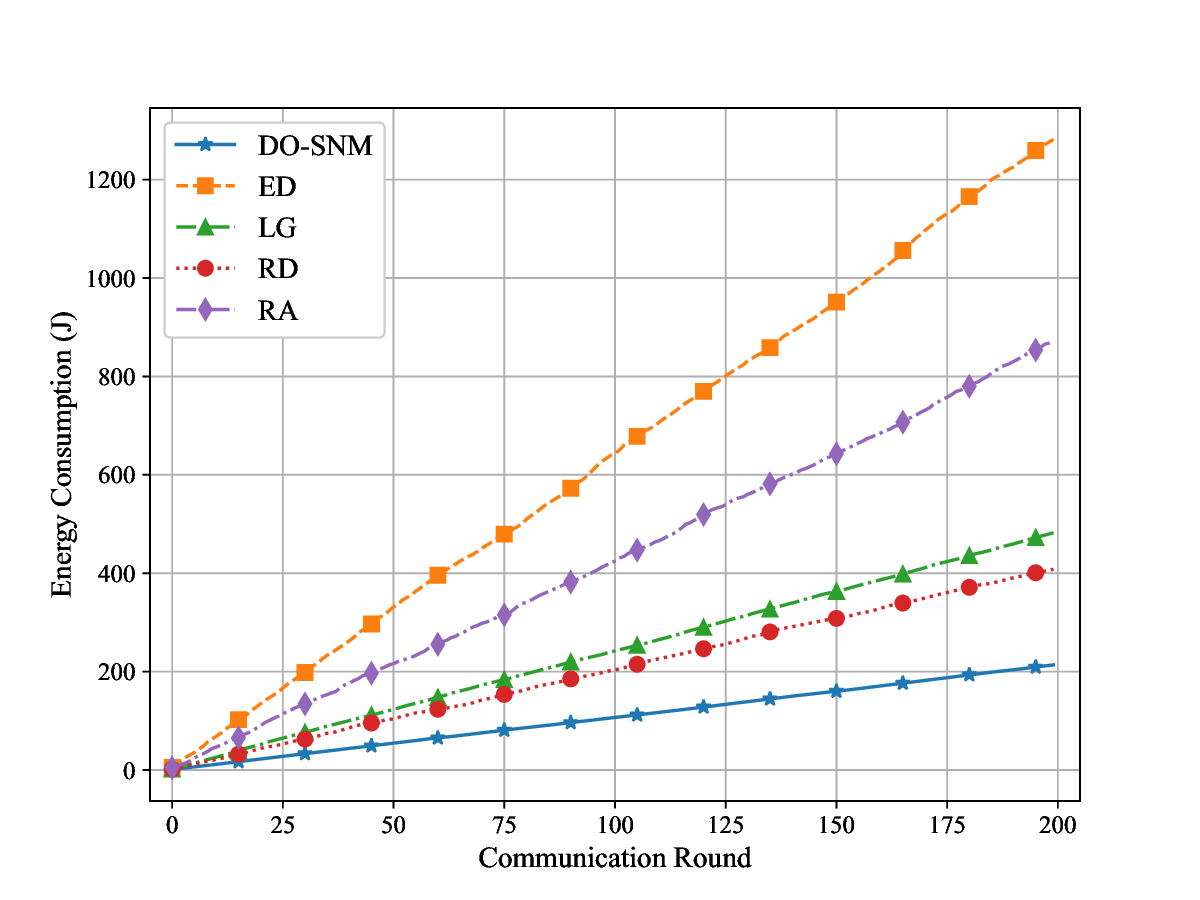}
        \end{minipage}
}
\caption{Test accuracy and energy consumption of CIFAR-10 and IEMOCAP datasets under different Algorithms.}
\end{figure*}

Considering the experimental results on CIFAR-10 as an instance, we observe distinct performance differentials among the baseline algorithms. The RA algorithm attains a test accuracy of $60\%$, $4\%$ lower than DO-SNM. Meanwhile, LG achieves $57\%$ accuracy, lagging $7\%$ behind our proposed method. Crucially, DO-SNM demonstrates superior energy efficiency, reducing consumption by approximately $60\%$ and $77\%$ compared to LG and RA, respectively. Both RA's random client selection and LG's proximity-based greedy strategy suffer from a critical limitation: fails to account for both the social network structure and client data-sharing relationships. This architectural oversight induces substantial variance in $\mathcal{G}_\mathcal{S}$ across global rounds, leading to significant performance fluctuations and eventual underfitting. In contrast, DO-SNM synthesizes graph structures with data characteristics to identify clients with minimized $\mathcal{G}_\mathcal{S}$. Given the positive correlation with $\mathcal{G}_\mathcal{S}$ inverse relationship with $\epsilon_\mathcal{S}$, our algorithm simultaneously elevates performance ceilings and achieves significant energy conservation.

Regarding RD and ED approaches, while both acknowledge client information sharing within the social network graph. RD prioritizes redundant data minimization, but achives only $56\%$ accuracy, $8\%$ lower than DO-SNM. ED maximizes effective data size, converging at $65\%$ accuracy, marginally $1\%$ higher than DO-SNM. In terms of energy consumption, RD consumes nearly the same energy as DO-SNM, while reducing ED’s consumption by over $80\%$. Although RD maintains low energ consumption, it leads to underfitting due to insufficient effective data. ED covers effective data over $80\%$ of the entire dataset, leading to prohibitive energy costs and rendering implementation impractical despite marginal accuracy advantages. Consequently, DO-SNM optimally balances model performance with energy efficiency, aligning with the core objectives of our joint optimization formulation \textbf{P}.

Additionally, due to the DP mechanism, Gaussian noise is injected during edge aggregation and broadcasting. Experimental results on the IEMOCAP dataset show that LG, RD, and RA exhibit significant overfitting beyond round 100. The overfitting stems from the constrained dataset scale and intrinsic label distribution imbalance fo IEMOCAP, further exacerbated by algorithmic instability that magnifies noise variance. Consequently, models overfit to the noise injected during training, resulting in significant performance deterioration. Conversely, ED maintains convergence stability through extensive client coverage, which dilutes noise impact on model parameters.

Collectively, these comparisons validate DO-SNM's efficiency under DP constraints, demonstrating synchronous optimization of energy efficiency, convergence speed, prediction accuracy, and noise robustness for practical deployment.
\section{Conclusion}
 In this paper, we integrate FL with social network scenario. Unlike the traditional FL structure, the information between clients might partially shared rather than entirely private. In this scenario, we introduce the concepts of EDCR and RDCR, and we examinethe impacts of effective data and redundant data on model performance. Considering the the mobility, social network and privacy protection, we develop the HFL-SNM framewor and formulate a joint optimization problem under resource and EDCR constraints. We decompose the problem into multiple subproblems and design DO-SNM algorithm to solve it. The core idea of the DO-SNM algorithm is to use two metric factors to evaluate the model performance and energy consumption, and combine efficient fast greedy algorithm and alternating optimization methods to address the three key sub-problems of client selection, edge association, and resource allocation layer by layer. 
 
 The experiment results shows that $r_\mathrm{ef}^0 =0.6$ and $\epsilon=80$ are the optimal trade-off points for model performance, energy consumption and privacy protection. Compared to the general baseline strategies RA and LG, our algorithm significantly reduces overhead by approximately $77\%$ and $60\%$ respectively. Compared to algorithms focused solely on effective or redundant data ED and RD, DO-SNM exhibits superior overall performance.
 
\appendices
\section{Preliminary experiments}
We explored the effect of the effective data and redundant data on model performance based on the CIFAR-10 dataset. 
\begin{figure*}[htbp]
\centering
\subfloat[Redundant data]
{
    \begin{minipage}[b]{.4\linewidth}
        \centering
        \includegraphics[scale=0.37]{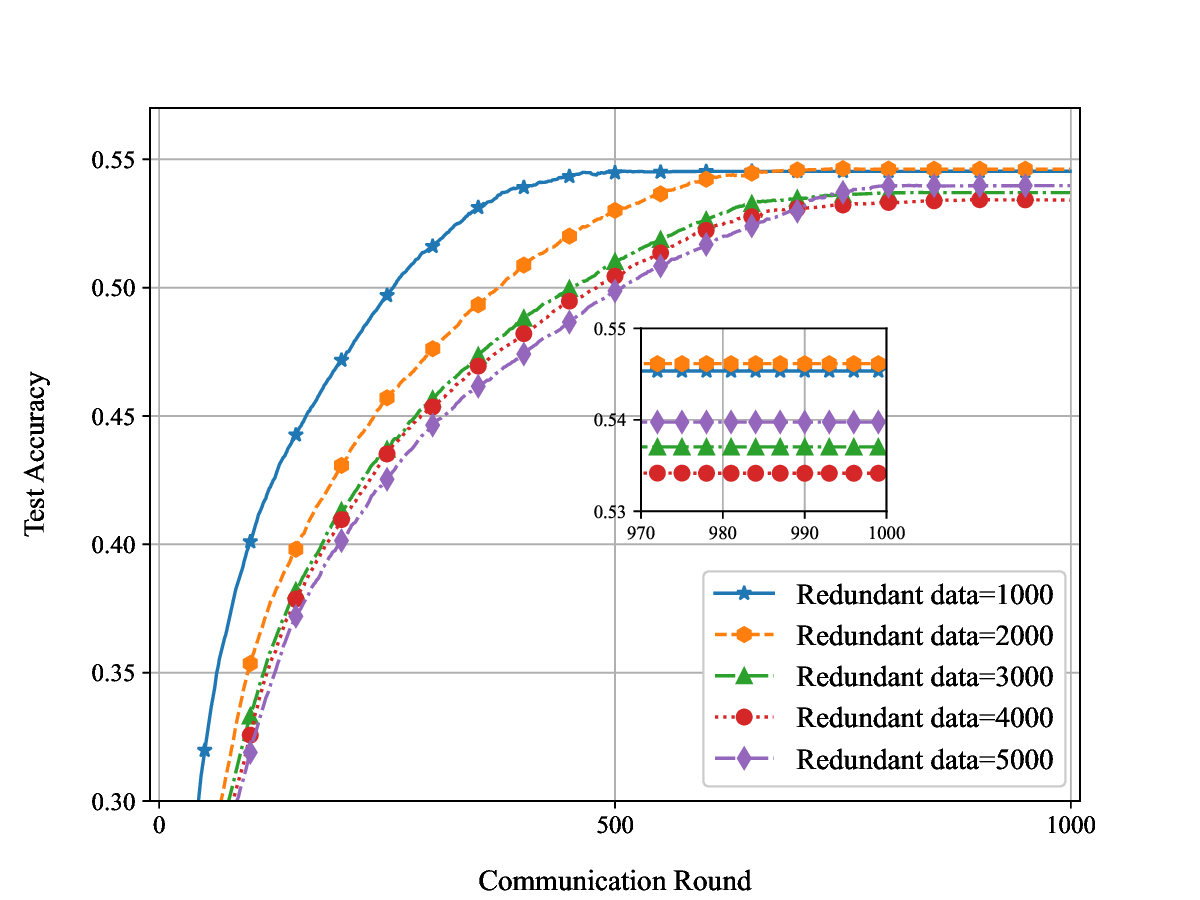}
        \end{minipage}
}
\subfloat[Effective data]
{
    \begin{minipage}[b]{.4\linewidth}
        \centering
        \includegraphics[scale=0.37]{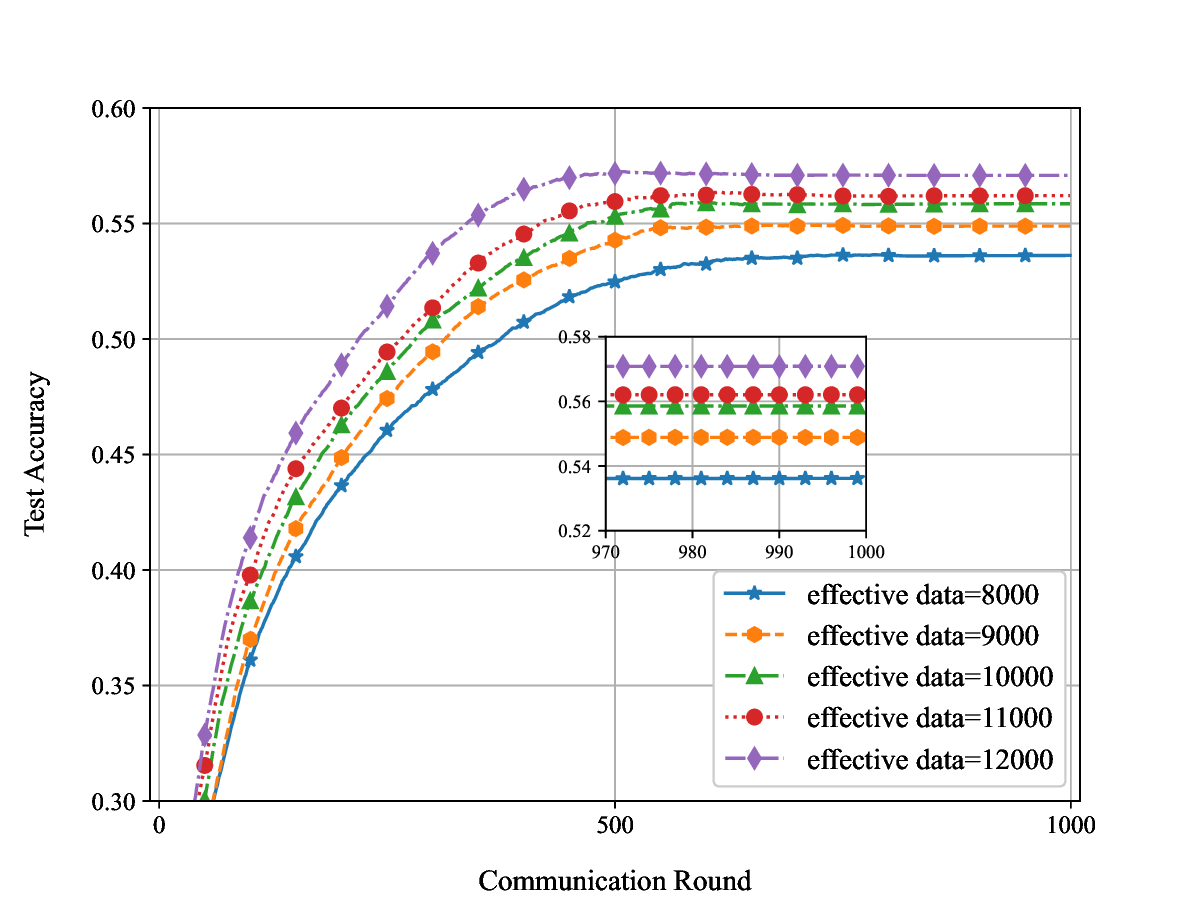}
        \end{minipage}
}
\caption{Impact of redundant data and effective data on model performance.}
\end{figure*}

For the experiment to explore the effect of redundant data on the model, we set the same amount of effective data ${D}_{\mathrm{ef}}^{\mathcal{S}} = 8000 $ for each set of samples, and the amount of redundant data ${D}_{\mathrm{re}}^{\mathcal{S}}$ varies from 1000 to 5000. As is shown in Fig. 6(a), before the 600-th global round, the model accuracy of the curve with the least redundant data consistently outperforms the other curves. As redundant data increases, the number of global rounds required for the model to converge increases. Moreover, the final model accuracy generally shows a downward trend. The result indicates that the model performance is negatively correlated with the redundant data. 

The other experiment explored the effect of effective data on the model performance. In this experiment, we set the same amount of redundant data ${D}_{\mathrm{re}}^{\mathcal{S}}=3000$ for each set of samples, varying the amount of effective data  ${D}_{\mathrm{ef}}^{\mathcal{S}}$ from 8000 to 12000. Fig. 6(b) clearly exhibits the impact of effective data on model performance. Throughout the entire process from training to convergence, as the effective data increases, the model accuracy increases from $53.8\%$ to $57\%$. Furthermore, the convergence speed gradually accelerates. The result indicates that the model performance is positively correlated with the effective data. 

The experimental results show that covering more effective data can improve model performance. However, large amounts of redundant data deteriorate model performance. Scheduling a large number of clients not only increases effective data but also inevitably increases redundant data, leading to significant energy consumption. While scheduling few clients leads to poor model performance. This makes it difficult to simultaneously achieve both low energy consumption and high model performance. Consequently, the processes of client selection and scheduling is pivotal in optimizing this balance.
\section{Proof of lemma 1}

The edge aggregation operation for $\mathcal{D}_i$ can be expressed as
\begin{align}
   s_{up}^{\mathcal{D}_i} &\stackrel{\triangle}{=} \boldsymbol{\omega}_i^\lambda = \underset{{\boldsymbol{\omega}}}{{\arg\min} \, } F_i(\boldsymbol{\omega},\mathcal{D}_{i}) \notag \\
   &= \frac{1}{D_i}\sum_{j=1}^{D_i} \underset{{\boldsymbol{\omega}}}{{\arg\min} \, } F_i(\boldsymbol{\omega},\mathcal{D}_{i,j}).
\end{align}
Let $\mathcal{D}_i'$ denote an adjacent dataset to $\mathcal{D}_i$, with the same data size as $\mathcal{D}_i$. Thus the sensitivity of $s_{e}^{\mathcal{D}_i}$ can be expressed as
\begin{align}
\Delta s_{e}^{\mathcal{D}_i} = &\max_{\mathcal{D}_i,\mathcal{D}_i'} \left\| s_{e}^{\mathcal{D}_i} - s_{e}^{\mathcal{D}_i'} \right\| \notag \\
= & \max_{\mathcal{D}_i,\mathcal{D}_i'} \Biggl\Vert \frac{1}{D_i}\sum_{j=1}^{D_i} \underset{{\boldsymbol{\omega}}}{{\arg\min}} \, F_i(\boldsymbol{\omega},\mathcal{D}_{i,j}) \notag \\
& - \frac{1}{D_i'}\sum_{j=1}^{D_i'} \underset{{\boldsymbol{\omega}}}{{\arg\min}} \, F_i(\boldsymbol{\omega},\mathcal{D'}_{i,j}) \Biggr\Vert \leq \frac{2W}{D_i'}.
\end{align} 

\section{Proof of lemma 2}
The sensitivity for $\mathcal{D}_n$ during the downlink phase can be expressed as
\begin{align}
    s_\mathrm{down}^{\mathcal{D}_n} &\stackrel{\triangle}{=} \boldsymbol{\omega_k} = \frac{1}{D^{\mathcal{S}^k}} \sum_{i=1}^{m_k}  D_i\boldsymbol{\omega_i^{\lambda_i}}
\end{align}
Similar to (30), the sensitivity can be given as
\begin{align}
\Delta s_\mathrm{down}^{\mathcal{D}_n} = &\max_{\mathcal{D}_n,\mathcal{D}_n'} \left\| s_\mathrm{down}^{\mathcal{D}_n} - s_\mathrm{down}^{\mathcal{D}_n'} \right\| \notag \\
= & \max_{\mathcal{D}_n,\mathcal{D}_n'} \Biggl\Vert \frac{D_n}{D^{\mathcal{S}^k}} \boldsymbol{\omega_n^{\lambda_n}}(D_n) -\frac{D_n'}{D^{\mathcal{S}^k}} \boldsymbol{\omega_n^{\lambda_n}}(D_n') \Biggr\Vert\notag \\
=& \frac{D_n}{D^{\mathcal{S}^k}} \max_{\mathcal{D}_n,\mathcal{D}_n'} \Biggl\Vert \boldsymbol{\omega_n^{\lambda_n}}(D_n) - \boldsymbol{\omega_n^{\lambda_n}}(D_n') \Biggl\Vert \notag \\
=& \frac{D_n \Delta s_\mathrm{up}^{\mathcal{D}_n}}{D^{\mathcal{S}^k}}  \leq \frac{2W}{D^{\mathcal{S}^k}}.
\end{align}
Furthermore, the aggregation process with artificial noise added by clients can be expressed
\begin{align}
    \tilde{\boldsymbol{\omega}}_k =& \frac{1}{D^{\mathcal{S}^k}} \sum_{i=1}^{m_k} D_i(\boldsymbol{\omega}_i^{\lambda_i}+\boldsymbol{n}_i) \notag \\
    =&\frac{1}{D^{\mathcal{S}^k}} (\sum_{i=1}^{m_k} D_i\boldsymbol{\omega}_i^{\lambda_i} + \sum_{i=1}^{m_k} D_i\boldsymbol{n}_i)
\end{align}

To ensure a edge $(\epsilon,\delta)$-DP in downlink channels, we know the standard deviation of additive noise can be set to $\sigma_k = \frac{cC^k\Delta s_\mathrm{down}^k}{\epsilon}$, where $\Delta s_\mathrm{down}^k = \frac{2W}{D^{\mathcal{S}^k}}$. Hence, we derive the standard derivation of additive noise at $e_k$ as
\begin{align}
    \sigma_\mathrm{down}^k = \sqrt{\sigma_k^2-\frac{1}{m_k^2} \sum_{i=1}^{m_k} {(}\sigma_\mathrm{up}^i)^2} =
\begin{cases}
\frac{2cW\sqrt{Q}}{\epsilon} & \mathrm{if} \phantom{k}Q>0 \\
\phantom{kk}0 & \mathrm{otherwise}
\end{cases}
\end{align}
\section{Proof of lemma 3}
We note that the function $\Lambda_n(\nu_{n}) = \frac{\alpha}{2} X_n \nu_{n}^2$ is convex with respect to $\nu_{n}$.  Let $x = \frac{B_{n,k}N_0}{h_{n,k}p_n}>0$. We define the function $\Psi(x)$ as  

\begin{align}
    \Psi(x) &= \frac{Y_n p_n}{B_{n,k}} \frac{h_{n,k}}{z N_0 \ln2} \notag \\
    &= \frac{1}{x \ln(1+\frac{1}{x})}.
\end{align}
Then the second derivative of $\Psi$ with respect to $x$ is derivied as 
\begin{equation}
    \frac{d^2 \Psi}{d x^2} = \frac{(x \ln(1+\frac{1}{x}))^2+2x^2 \ln(1+\frac{1}{x})(1-(x+1) \ln(1+\frac{1}{x}))^2}{x(x+1)^2(x \ln(1+\frac{1}{x}))^4}.
\end{equation}
Evidently, the result of (38) is positive. Consquently, $\Psi$ is convex with respect to $x$ and equivalently to $B_{n,k}$. Since $E_k$ can be expressed as an affine combination of $\Lambda(\nu_{n})$ and $\Psi(x)$, and since affine transformations preserve convexity, it follows that $E_k$ is also a convex function. 

\section{Proof of lemma 4}
The partial derivatives of (21) are computed by 
\begin{subequations}
\begin{align}\label{eq:(179)}
&\frac{\partial L_k}{\partial \nu_n}=\tau \alpha X_n\nu_n-\gamma_n + \sigma_n - \frac{\theta_nX_n}{\nu_n^2}, \\
&\frac{\partial L_k}{\partial B_{n,k}}=\frac{\tau p_n Y_n^2}{z\ln2 B_{n,k}^2}(\frac{1}{1+\frac{B_{n,k}N_0}{h_{n,k}p_n}}-\ln (1+\frac{h_{n,k}p_n}{B_{n,k}N_0}))+\mu_k \notag\\
&\phantom{\frac{\partial L_k}{\partial B_{n,k}}}+\frac{\theta_n Y_n^2}{z\ln2 B_{n,k}^2}(\frac{1}{1+\frac{B_{n,k}N_0}{h_{n,k}p_n}}-\ln (1+\frac{h_{n,k}p_n}{B_{n,k}N_0})).
\end{align}
\end{subequations}
Furthermore, by setting (37a) and (37b) equal to zero and incorporating the constraint conditions, we obtain the following equations.
\begin{equation}\label{eq:(199)}
\left\{
\begin{aligned}
&\tau \alpha X_n\nu_n-\gamma_n + \sigma_n - \frac{\theta_nX_n}{{\nu_n}^2} =0, && \text{(a)} \\
&\frac{\tau p_n Y_n^2}{z\ln2 B_{n,k}^2}(\frac{1}{1+\frac{B_{n,k}N_0}{h_{n,k}p_n}}-\ln (1+\frac{h_{n,k}p_n}{B_{n,k}N_0}))+\mu_k \\
&+\frac{\theta_n Y_n^2}{z\ln2 B_{n,k}^2}(\frac{1}{1+\frac{B_{n,k}N_0}{h_{n,k}p_n}}-\ln (1+\frac{h_{n,k}p_n}{B_{n,k}N_0}))=0, && \text{(b)} \\
&\mu_k(\sum_{n \in \{i| u_i \in \mathcal{S}^k\}}B_{n,k}-B_0) =0, && \text{(c)} \\
&\theta_n(\frac{X_n}{\nu_n}+\frac{Y_n}{B_{n,k}}-t_0) =0, && \text{(d)} \\
&\gamma_n(\nu_\mathrm{min}-{\nu_n})=0, && \text{(e)} \\
&\sigma_n({\nu_n}-\nu_\mathrm{max})=0, && \text{(f)}
\end{aligned}
\right.
\end{equation}
 where (38c) - (38f) are the complementary slackness conditions of \textbf{P}. We obtain (22a) and (22b) by rearranging the terms of (38a) and (38b) and derive $\mu_k>0$. Thus, (38c) becomes a binding constraint, which consequently yields (22c).

\section{Proof of lemma 5}
Assume, for the sake of contradiction, that the original statement is false. We have $\nu_{min}<\nu_n \leq \nu_{max} $ if constraint (20a) is slack for $u_n$. Let us consider the following two cases:
\begin{itemize}
    \item \textbf{$\nu_{min}<\nu_n < \nu_{max}$:}  According to (38e) (38f) and (22a), we know $\gamma_n = 0$, $\sigma_n=0$ and $\theta_n = \tau\alpha \nu_n^3 > 0$. By referring to (38d), we konw $\frac{X_n}{\nu_n}+\frac{Y_n}{B_{n,k}}-t_0 = 0 $. Hence, the constraint (16a) is binding, which is a contradiction because the constraint (16a) is slack. Therefore, this case is false.
    \item \textbf{$\nu_n = \nu_{max}$:}  According to (38e) (38f) and (22a), we know $\gamma_n = 0$, $\sigma_n \geq 0$ and $\theta_n = \tau\alpha(\nu_n)^3+\frac{\sigma_n(\nu_n)^2}{X_n} \geq \tau\alpha \nu_n^3 > 0$. By referring to (38d), we konw $\frac{X_n}{\nu_n}+\frac{Y_n}{B_{n,k}}-t_0 = 0 $. Hence, the constraint (20a) is binding, which is a contradiction because the constraint (20a) is slack. Therefore, this case is false.
\end{itemize}
Therefore, our assumption is false and the original statements in lemma2 make sense.

\section{Proof of lemma 6}
According to (38c), (22b) and (22c), if the frequencies of clients are given, the bandwidth distribution can be derived. The equation of (22b) is rewrite as
\begin{equation}
    \frac{z \mu_k \ln2}{\tau p_n + \theta_n}=(\frac{Y_n}{B_{n,k}})^2(\ln (1+\frac{h_{n,k}p_n}{B_{n,k}N_0}) - \frac{1}{1+\frac{B_{n,k}N_0}{h_{n,k}p_n}}).
\end{equation}
Then we discuss discuss how the bandwidth allocation and $\mu_k$ vary with frequencies of different clients.  Let $x = \frac{B_{n,k}N_0}{h_{n,k}p_n}>0$. We define $\Phi_n = \ln(1+\frac{1}{x}) - \frac{1}{1+x}$. The first derivative of $\Phi_n$ with respect to $x$ is derived as
    \begin{equation}
        \frac{d\Phi}{dx} = -\frac{1}{x(x+1)^2} < 0.
    \end{equation} 
Then $\Phi_n > 0$ is monotonically decreasing with respect to $x$ and equivalently to $B_{n,k}$. Moreover, $(\frac{Y_n}{B_{n,k}})^2 > 0$ is monotonically decreasing with respect to $B_{n,k}$. We have the right side of (39) is monotonically decreasing with respect to $B_{n,k}$. Considering the following cases :
\begin{itemize}
    \item \textbf{$B_{n,k}$ remains constant:} In this case, the right side of the equation (39) remains unchanged, and consequently, the left side of the equation must also be unchanged. If $\nu_n$ increases, $\mu_k$ increases. For clients $u_i \neq u_n$, $\mu_k$ remains unchanged. While all clients in $\mathcal{S}^k$ have equal $\mu_k$. Thus this case is false.
    \item \textbf{$B_{n,k}$ is monotonically decreasing with respect to $\nu_n$:} We assume $\nu_n$ increases. Then $B_{n,k}$ decreases and  $(\frac{Y_n}{B_{n,k}})^2\Phi_n$ increases, which means $\mu_k$ must increase. For clients $u_i \neq u_n$, $\nu_i$ remains unchanged. Then $(\frac{Y_i}{B_{i,k}})^2\Phi_i$ increases and $B_{i,k}$ decreases. We derive $\sum_{n \in \{i| u_i \in \mathcal{S}^k\}}B_{n,k}^* < B_0$, which is contradictory with (22c).Thus this case is also false.
    \item \textbf{$B_{n,k}$ is monotonically increasing with respect to $\nu_n$:} Similarly, we assume $\nu_n$ increases. Then $B_{n,k}$ increases and  $(\frac{Y_n}{B_{n,k}})^2\Phi_n$ decreases. We need to further discussion for $\mu_k$ to judge the feasibility of this case.
    \begin{itemize}
        \item \textbf{$\mu_k$ remains constant:} In this case, For clients $u_i \neq u_n$, $(\frac{Y_i}{B_{i,k}})^2\Phi_i$ and $B_{i,k}$ remains unchanged. We have $\sum_{n \in \{i| u_i \in \mathcal{S}^k\}}B_{n,k} > B_0$, which is contradictory with (22c). This case is false.
        \item \textbf{$\mu_k$ decrease:} In this case, For clients $u_i \neq u_n$, $(\frac{Y_i}{B_{i,k}})^2\Phi_i$ decreases and $B_{i,k}$ increase. We have $\sum_{n \in \{i| u_i \in \mathcal{S}^k\}}B_{n,k} > B_0$, which is contradictory with (22c). This case is also false.
        \item \textbf{$\mu_k$ increase:} In this case, For clients $u_i \neq u_n$, $(\frac{Y_i}{B_{i,k}})^2\Phi_i$ increases and $B_{i,k}$ decrease. Because $B_{n,k}$ increases, (22c) is solvable. Therefore, this case makes sense.
    \end{itemize}
\end{itemize}
Consequently, $B_{n,k}$ is is monotonically increasing with respect to $\nu_n$ while decreasing with respect to $\nu_i, i \neq n$. Combined with $\Gamma_n = \frac{X_n}{\nu_n}+ \frac{z}{B_{n,k}\ln (1+\frac{h_{n,k}p_n}{B_{n,k}B_0})} - t_0$, we easily derive that $\Gamma_n$ is monotonically decreasing with respect to $\nu_n$ while increasing with respect to $\nu_i, i \neq n$.

\section{Proof of lemma 7}
According to lemma3, we have 
$\Gamma_n(\nu_n=\nu_{min},\nu_i = \nu_i^*, i \neq n) \geq \Gamma_n(\nu_n=\nu_{min},\nu_i = \nu_{min}, i \neq n) 
> 0 
\geq \Gamma_n(\nu_n=\nu_n^*,\nu_i = \nu_i^*, i \neq n)$.    
Consequently, $\nu_n^* >\nu_{min}$ and the constraint (20a) is binding on $u_n$.
\bibliographystyle{IEEEtran}

\begin{IEEEbiography}[{\includegraphics[width=1in,height=1.25in,clip,keepaspectratio]{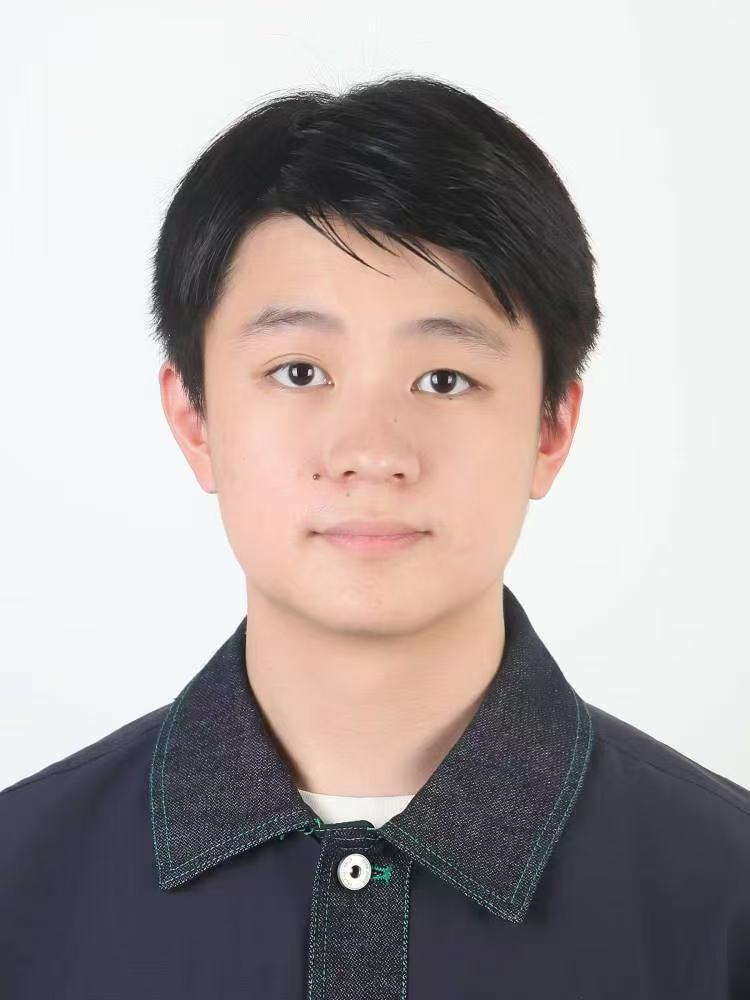}}]{Zeyu Chen}
received the B.S. degree in energy engineering from Huazhong University of Science and Technology (HUST) in 2023. He is currently pursuing the M.S. degree with the Broad-band Access Network Laboratory, Department of Electronic Engineering, Shanghai Jiao Tong University (SJTU), China. His research interests include federated learning and wireless networks.
\end{IEEEbiography}
\begin{IEEEbiography}[{\includegraphics[width=1in,height=1.25in,clip,keepaspectratio]{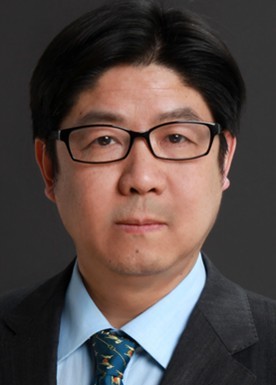}}]{Wen Chen}
(M’03–SM’11) received BS and MS from Wuhan University, China in 1990 and 1993 respectively, and PhD from University of Electro-communications, Japan in 1999. He is now a tenured Professor with the Department of Electronic Engineering, Shanghai Jiao Tong University, China. He is a fellow of Chinese Institute of Electronics and the distinguished lecturers of IEEE Communications Society and IEEE Vehicular Technology Society. He also received Shanghai Natural Science Award in 2022. He is the Shanghai Chapter Chair of IEEE Vehicular Technology Society, a vice president of Shanghai Institute of Electronics, Editors of IEEE Transactions on Wireless Communications, IEEE Transactions on Communications, IEEE Access and IEEE Open Journal of Vehicular Technology. His research interests include multiple access, wireless AI and RIS communications. He has published more than 200 papers in IEEE journals with citations more than11,000 in Google scholar. 
\end{IEEEbiography}
\begin{IEEEbiography}[{\includegraphics[width=1in,height=1.25in,clip,keepaspectratio]{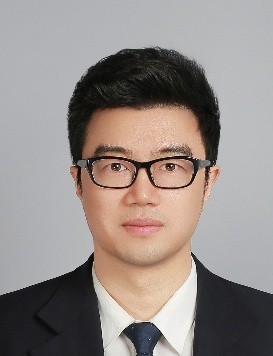}}]{Jun Li} 
(M’09-SM’16-F’25) received Ph.D. degree in Electronic Engineering from Shanghai Jiao Tong University, Shanghai, P. R. China in 2009. From January 2009 to June 2009, he worked in the Department of Research and Innovation, Alcatel Lucent Shanghai Bell as a Research Scientist. From June 2009 to April 2012, he was a Postdoctoral Fellow at the School of Electrical Engineering and Telecommunications, the University of New South Wales, Australia. From April 2012 to June 2015, he was a Research Fellow at the School of Electrical Engineering, the University of Sydney, Australia. From June 2015 to June 2024, he was a Professor at the School of Electronic and Optical Engineering, Nanjing University of Science and Technology, Nanjing, China. He is now a Professor at the School of Information Science and Engineering, Southeast University, Nanjing, China. He was a visiting professor at Princeton University from 2018 to 2019. His research interests include distributed intelligence, multiple agent reinforcement learning, and their applications in ultra-dense wireless networks, mobile edge computing, network privacy and security, and industrial Internet of Things. He has co-authored more than 300 papers in IEEE journals and conferences. He was serving as an editor of IEEE Transactions on Wireless Communication and TPC member for several flagship IEEE conferences.
\end{IEEEbiography}
\begin{IEEEbiography}
[{\includegraphics[width=1in,height=1.25in,clip,keepaspectratio]{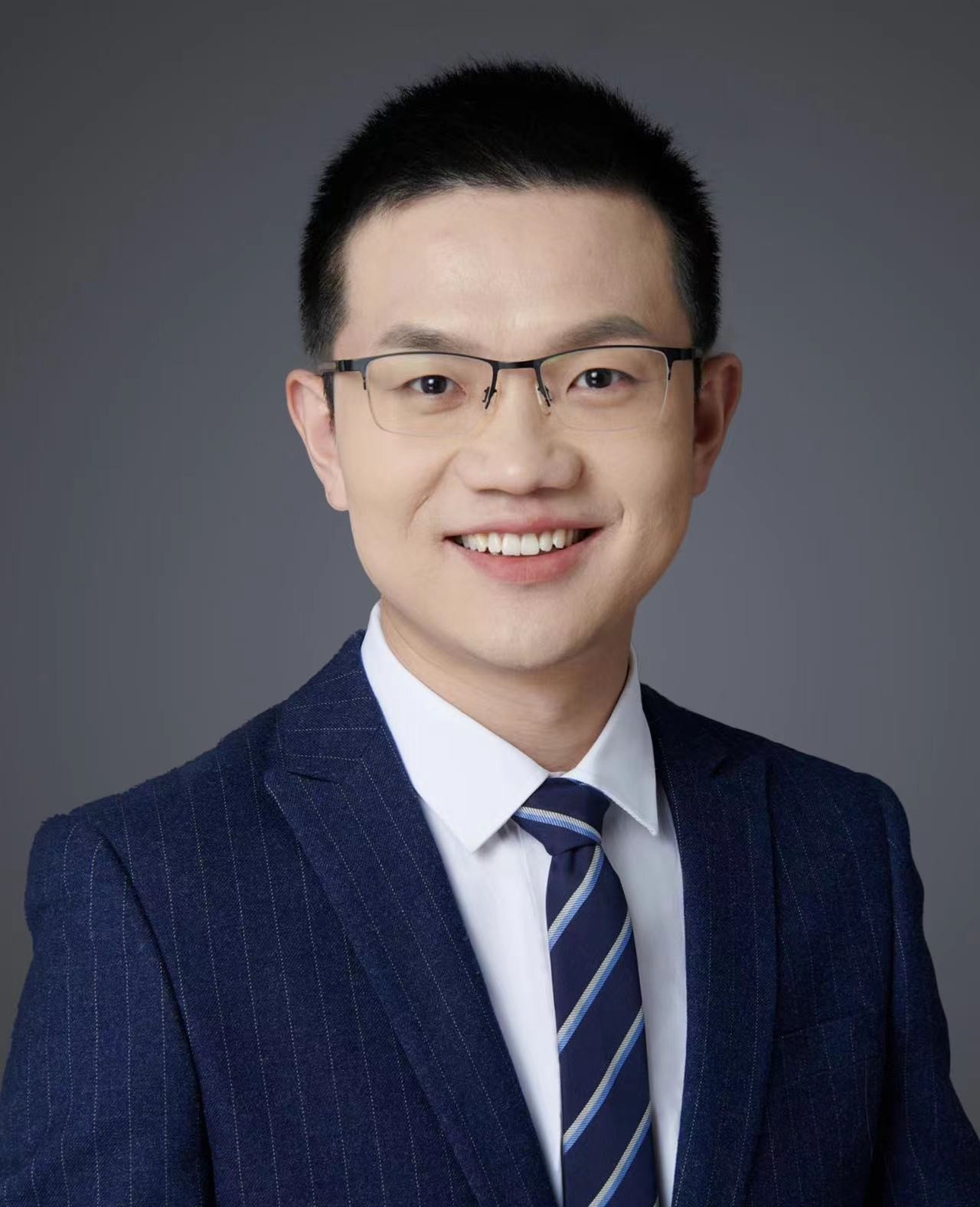}}]{Qingqing Wu}
(S’13-M’16-SM’21) is an Associate Professor with Shanghai Jiao Tong University. His current research interest includes intelligent reflecting surface (IRS), unmanned aerial vehicle (UAV) communications, and MIMO transceiver design. He has coauthored more than 100 IEEE journal papers with 40+ ESI highly cited papers and 10+ ESI hot papers, which have received more than 41,000 Google citations. He has been listed as the Clarivate ESI Highly Cited Researcher since 2021, the Most Influential Scholar Award in AI-2000 by Aminer since 2021, World’s Top 2
He was the recipient of the IEEE ComSoc Fred Ellersick Prize, Best Tutorial Paper Award in 2023, Asia-Pacific Best Young Researcher Award and Outstanding Paper Award in 2022, Young Author Best Paper Award in 2021 and 2024, the Outstanding Ph.D. Thesis Award of China Institute of Communications in 2017, the IEEE ICCC Best Paper Award in 2021, and IEEE WCSP Best Paper Award in 2015. He was the Exemplary Editor of IEEE Communications Letters in 2019 and the Exemplary Reviewer of several IEEE journals. He serves as an Associate/Senior/Area Editor for IEEE Transactions on Wireless Communications, IEEE Transactions on Communications, IEEE Communications Letters, IEEE Wireless Communications Letters. He is the Lead Guest Editor for IEEE Journal on Selected Areas in Communications. He is the workshop co-chair for IEEE ICC 2019-2023 and IEEE GLOBECOM 2020. He serves as the Workshops and Symposia Officer of Reconfigurable Intelligent Surfaces Emerging Technology Initiative and Research Blog Officer of Aerial Communications Emerging Technology Initiative. He has served as the Chair of the IEEE ComSoc Young Professional
\end{IEEEbiography}
\begin{IEEEbiography}
[{\includegraphics[width=1in,height=1.25in,clip,keepaspectratio]{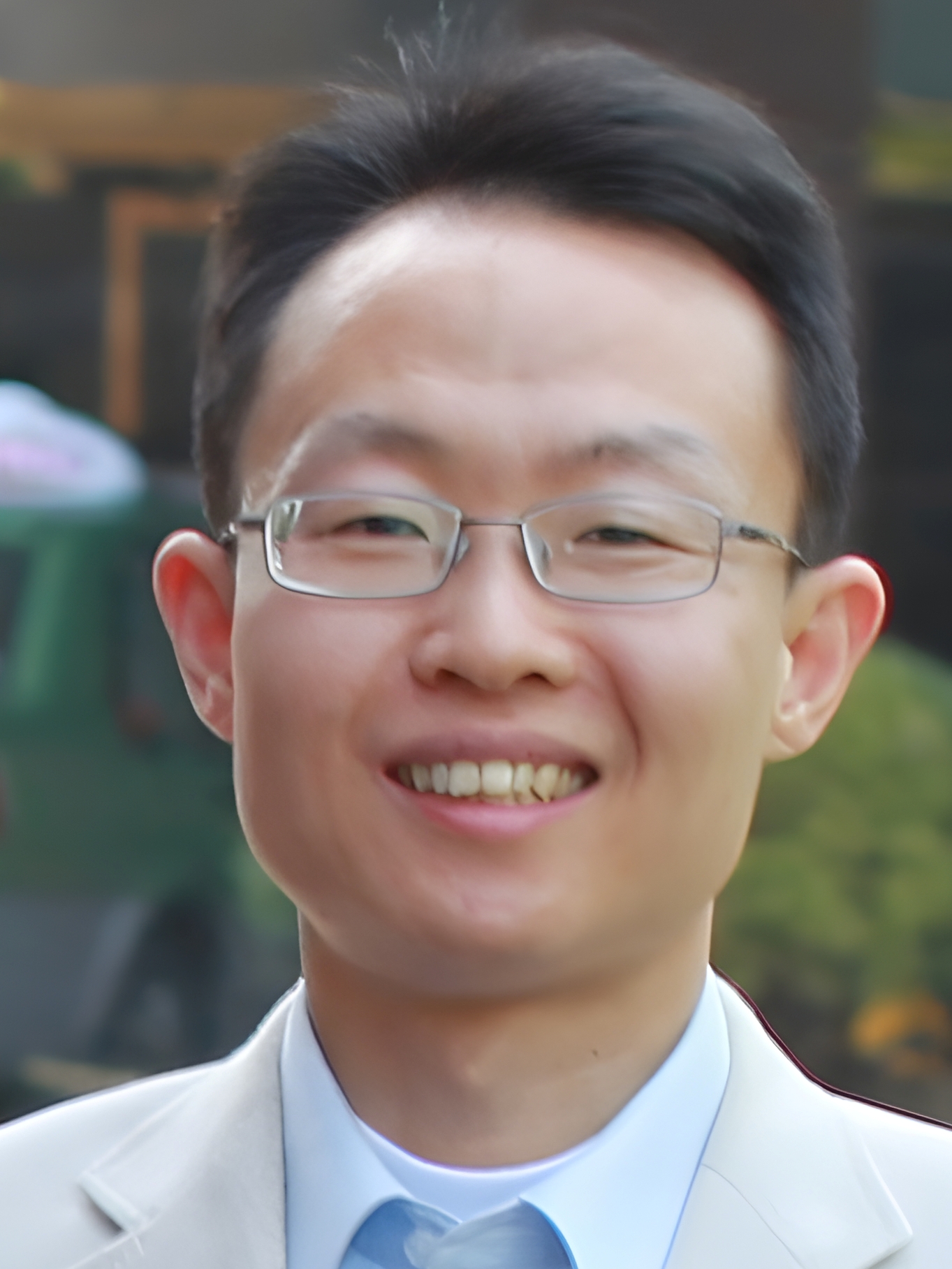}}]{Ming Ding}
(IEEE M’12-SM’17) received the B.S. (with first-class Hons.) and M.S. degrees in electronics engineering from Shanghai Jiao Tong University (SJTU), China, and the Doctor of Philosophy (Ph.D.) degree in signal and information processing from SJTU, in 2004, 2007, and 2011, respectively. From April 2007 to September 2014, he worked at Sharp Laboratories of China as a Researcher/Senior Researcher/Principal Researcher. Currently, he is the Group Leader of the Privacy Technology Group at CSIRO’s Data61 in Sydney, NSW, Australia. Also, he is an Adjunct Professor at Swinburne University of Technology and University of Technology Sydney, Australia. His research interests include data privacy and security, machine learning and AI, and information technology. He has co-authored more than 250 papers in IEEE/ACM journals and conferences, all in recognized venues, and around 20 3GPP standardization contributions, as well as two books, i.e., “Multi-point Cooperative Communication Systems: Theory and Applications” (Springer, 2013) and “Fundamentals of Ultra-Dense Wireless Networks” (Cambridge University Press, 2022). Also, he holds 21 US patents and has co-invented another 100+ patents on 4G/5G technologies. Currently, he is an editor of IEEE Communications Surveys and Tutorials and IEEE Transactions on Network Science and Engineering. Besides, he has served as a guest editor/co-chair/co-tutor/TPC member for multiple IEEE top-tier journals/conferences and received several awards for his research work and professional services, including the prestigious IEEE Signal Processing Society Best Paper Award in 2022 and Highly Cited Researcher recognized by Clarivate Analytics in 2024.
\end{IEEEbiography}
\begin{IEEEbiography}
[{\includegraphics[width=1in,height=1.25in,clip,keepaspectratio]{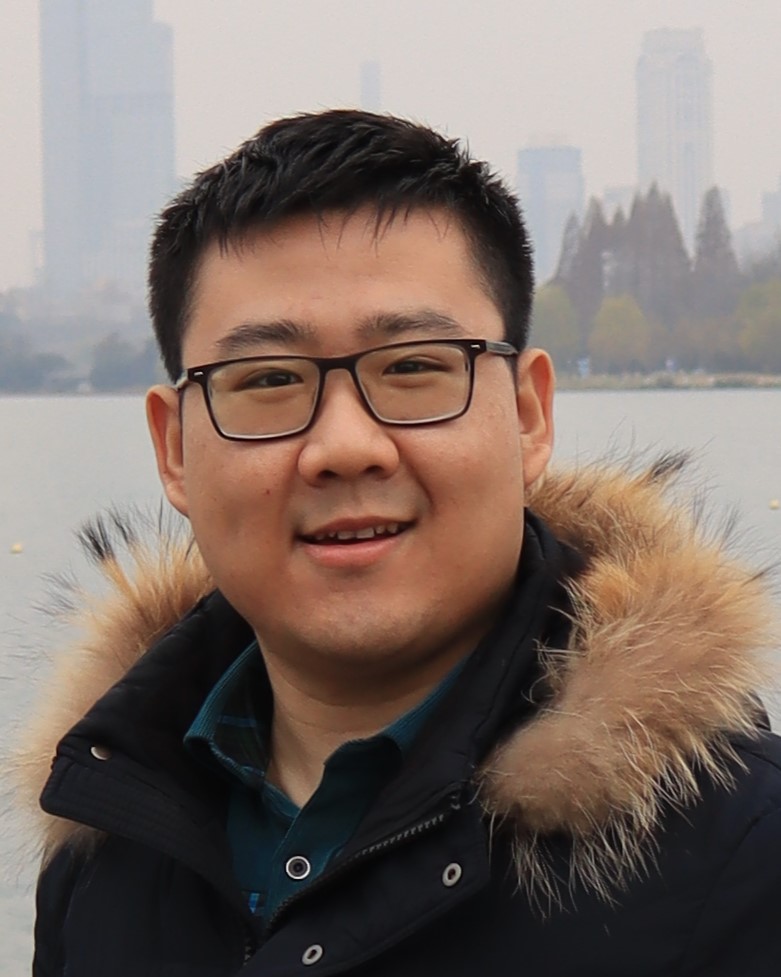}}]{Xuefeng Han}
received the B.S. degree in communication engineering from University of Electronic Science and Technology of China in 2020, and received the Ph.D. degree from Department of Electronic Engineering, Shanghai Jiao Tong University in 2025. His research interests include federated learning, lightweight neural network, multimodal learning and resource management in future wireless networks.
\end{IEEEbiography}
\begin{IEEEbiography}
[{\includegraphics[width=1in,height=1.25in,clip,keepaspectratio]{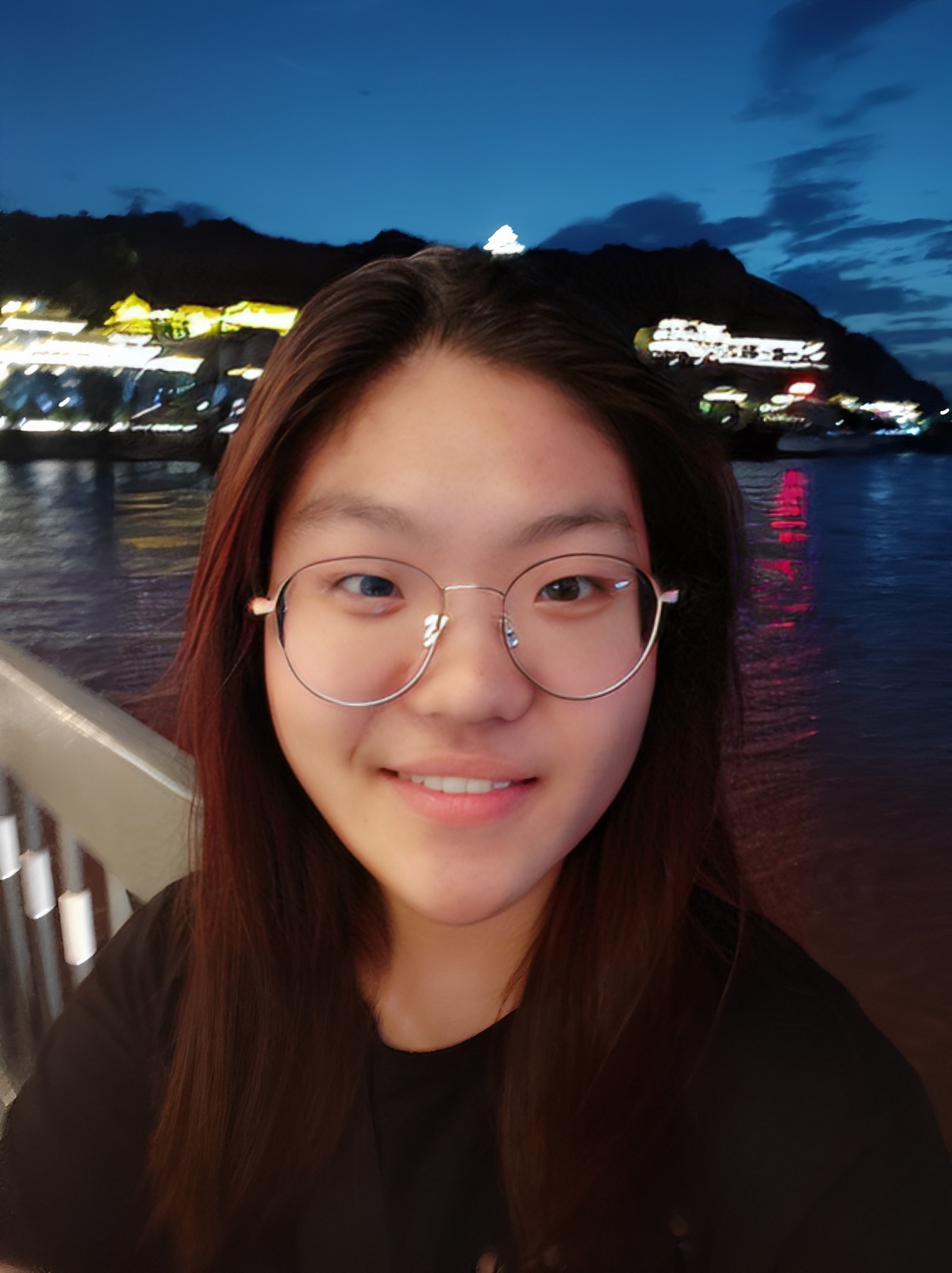}}]
	{Xiumei Deng} received the Ph.D. degree in Information and Communications Engineering from Nanjing University of Science and Technology, Nanjing, China, in 2024. Before that, she received the B.E. degree in Electronic Information Engineering from Nanjing University of Science and Technology, Nanjing, China, in 2018. She is currently a postdoctoral fellow at the Singapore University of Technology and Design.
\end{IEEEbiography}
\begin{IEEEbiography}
[{\includegraphics[width=1in,height=1.25in,clip,keepaspectratio]{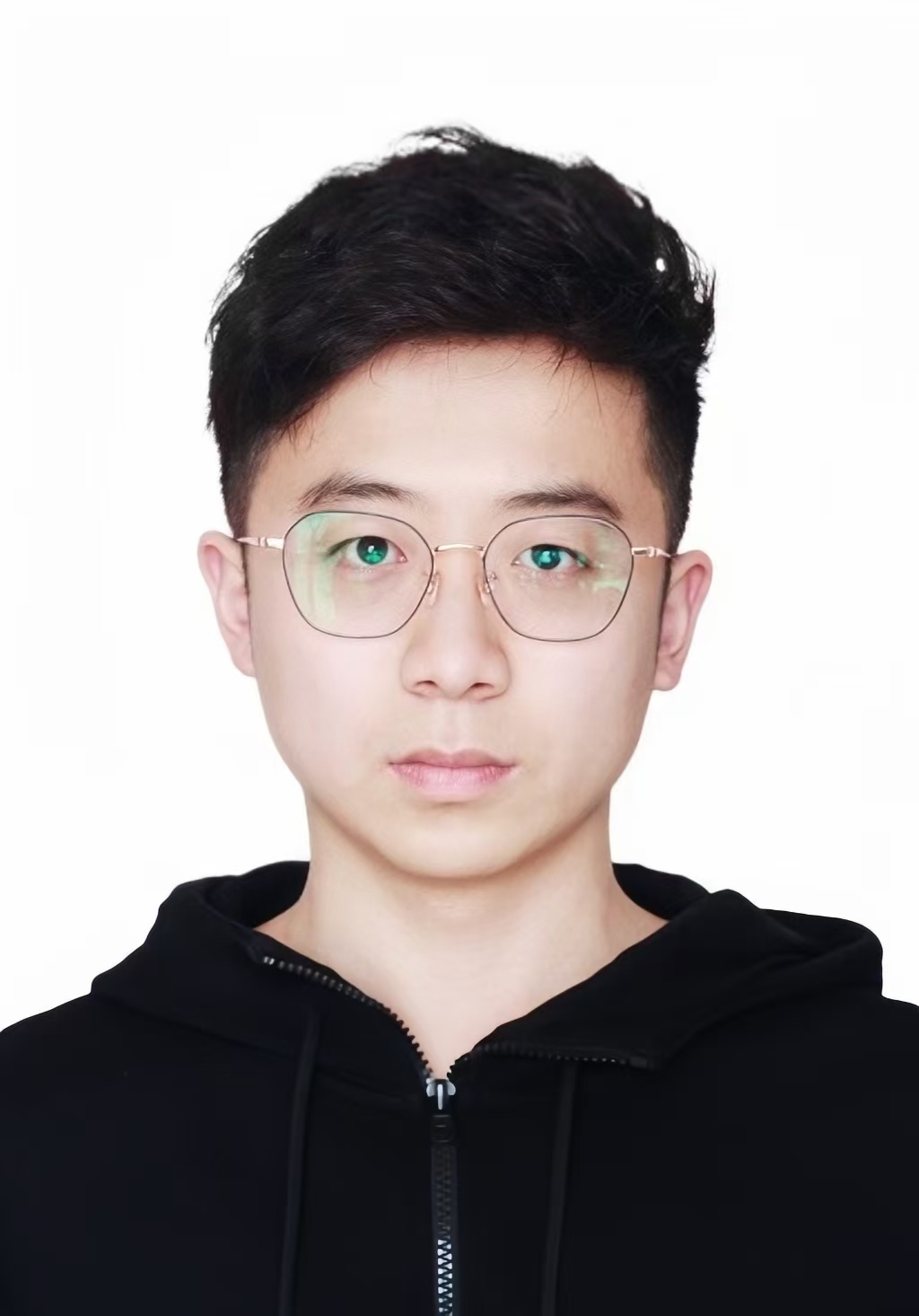}}]
	{Liwei Wang} is currently pursuing the Ph.D. degree with the Broad-band Access Network Laboratory, Department of Electronic Engineering, Shanghai Jiao Tong University (SJTU), Shanghai, China. His research interests include federated learning and wireless resource management in wireless networks.
\end{IEEEbiography}
\end{document}